\documentclass{article}
\usepackage[utf8]{inputenc}

\usepackage{arxiv}

\usepackage[utf8]{inputenc} 
\usepackage[T1]{fontenc}    
\usepackage{hyperref}       
\usepackage{booktabs}       
\usepackage{amsfonts}       
\usepackage{nicefrac}       
\usepackage{microtype}      
\usepackage[dvipsnames]{xcolor}
\usepackage{wrapfig}
\usepackage{enumitem}
\usepackage{multirow} 

\usepackage{amsmath}
\usepackage{amssymb}
\usepackage{amsthm}
\usepackage{bbm} 
\usepackage{stmaryrd}

\usepackage{natbib}
\setcitestyle{authoryear,open={(},close={)}}
\usepackage[ruled]{algorithm2e}
\makeatletter
\newcounter{framework}
\newenvironment{framework}[1][htb]
  {
   \let\c@algocf\c@framework
   \begin{algorithm}[#1]%
  }{\end{algorithm}}

\usepackage{thm-restate}

\usepackage{tikz}
\usetikzlibrary{automata, positioning, arrows.meta}

\DeclareMathOperator*{\argmin}{arg\,min}
\DeclareMathOperator*{\argmax}{arg\,max}
\newcommand{\R}{\mathbb{R}}
\newcommand{\N}{\mathbb N}
\newcommand{\E}{\mathbb{E}}
\newcommand{\Var}{\mathbb{V}\mathrm{ar}}
\newcommand{\bbP}{\mathbb{P}}
\newcommand{\bbQ}{\mathbb{Q}}

\newcommand{\cH}{\mathcal{H}}

\newcommand{\cF}{\mathcal{F}}
\newcommand{\cG}{\mathcal{G}}

\newcommand{\cC}{\mathcal{C}}
\newcommand{\pr}{\mathbb{P}}

\newcommand{\diam}{\mathrm{diam}}

\newcommand{\Reg}{\mathrm{Reg}}

\newcommand{\cM}{\mathcal{M}}

\newcommand{\cS}{\mathcal{S}}
\newcommand{\cA}{\mathcal{A}}
\newcommand{\cP}{\mathcal{P}}

\newcommand{\sumT}{\sum_{t=1}^T}
\newcommand{\sumH}{\sum_{h=1}^H}

\newcommand{\llVert}{\left\lVert}
\newcommand{\rrVert}{\right\rVert}
\newcommand{\htheta}{\hat\theta}

\newcommand{\MDP}{\mathrm{MDP}}

\makeatletter
\newcommand*\rel@kern[1]{\kern#1\dimexpr\macc@kerna}
\newcommand*\widebar[1]{%
  \begingroup
  \def\mathaccent##1##2{%
    \rel@kern{0.8}%
    \overline{\rel@kern{-0.8}\macc@nucleus\rel@kern{0.2}}%
    \rel@kern{-0.2}%
  }%
  \macc@depth\@ne
  \let\math@bgroup\@empty \let\math@egroup\macc@set@skewchar
  \mathsurround\z@ \frozen@everymath{\mathgroup\macc@group\relax}%
  \macc@set@skewchar\relax
  \let\mathaccentV\macc@nested@a
  \macc@nested@a\relax111{#1}%
  \endgroup
}
\makeatother

\renewcommand{\bar}{\widebar}
\renewcommand{\tilde}{\widetilde}
\renewcommand{\hat}{\widehat}
\renewcommand{\le}{\leq}
\renewcommand{\epsilon}{\varepsilon}

\makeatletter
\newcommand{\Biggg}[1]{{\hbox{$\left#1\vbox to 25\p@{}\right.\n@space$}}}
\newcommand{\Bigggg}[1]{{\hbox{$\left#1\vbox to 40\p@{}\right.\n@space$}}}
\makeatother

\newif\ifinappendix
\newcommand{\mainonly}[1]{\ifinappendix\else#1\fi}
\newcommand{\appendixonly}[1]{\ifinappendix#1\fi}

\title{
Minimax Optimal Variance-Aware Regret Bounds for Multinomial Logistic MDPs}
\author{%
  Pierre Boudart \\
  INRIA, École Normale Supérieure\\
  CNRS, PSL Research University\\
  Paris, France\\
  \texttt{pierre.boudart@inria.fr} \\
  \And
  Pierre Gaillard \\
  Univ. Grenoble Alpes, Inria, \\
  CNRS, Grenoble INP, LJK \\
  Grenoble, France\\
  \texttt{pierre.gaillard@inria.fr} \\
  \AND
  Alessandro Rudi \\
  SDA Bocconi School of Management\\
  Milano, Italy\\
  \texttt{alessandro.rudi@sdabocconi.it}
}

\begin{document}

\maketitle

\begin{abstract}
We study reinforcement learning for episodic Markov Decision Processes (MDPs) 
whose transitions are modelled by a multinomial logistic (MNL) model. 
Existing algorithms for MNL mixture MDPs yield a regret of 
$\smash{\tilde{O}(dH^2\sqrt{T})}$ \citep{li2024provably}, where $d$ is the feature 
dimension, $H$ the episode length, and $T$ the number of episodes. 
Inspired by the logistic bandit literature \citep{abeille2021instance, 
faury2022jointly, boudart2026enjoying}, we introduce a problem-dependent 
constant $\bar\sigma_T \leq 1/2$, measuring the normalised average variance 
of the optimal downstream value function along the learner's trajectory. 
We propose an algorithm achieving a regret of 
$\smash{\tilde{O}(dH^2\bar\sigma_T\sqrt{T})}$, which recovers the existing bound 
in the worst case and improves upon it for structured MDPs. For instance, for KL-constrained robust MDPs, $\bar\sigma_T = O(H^{-1})$, reducing the horizon dependence by a factor $H$. 
We further establish a matching $\smash{\Omega(dH^2\bar\sigma_T\sqrt{T})}$ lower bound, proving minimax optimality (up to logarithmic factors) and fully characterising the regret complexity of 
MNL mixture MDPs for the first time.
\end{abstract}

\section{Introduction}\label{section:introduction}

Reinforcement Learning (RL) is a control-theoretic problem where a learner aims to maximise its expected cumulative reward by sequentially interacting with an unknown environment \citep{sutton1998reinforcement}.
In practice, modern RL tackles problems with large numbers of states, requiring function approximation to approximate the state-action value function -- the expected cumulative reward starting from a state-action pair -- or the policy -- the mapping from states to actions.
RL with function approximation has been very successful in different applications including Backgammon \citep{tesauro1994td}, Atari games \citep{mnih2015human}, Go \citep{silver2017mastering}, robotics \citep{kober2013reinforcement}, and Dialogue generation \citep{li2016deep}.
Formally, the learner interacts with the environment over $T$ episodes, each consisting of $H$ steps. At each step $h$ of episode $t$, the learner observes a state $s_{t,h}\in\cS$, selects an action $a_{t,h}\in\cA$ according to a chosen policy $\pi_t$, receives a reward $\smash{r_h(s_{t,h}, a_{t,h}) }$, and transitions to a next state $\smash{ s_{t,h+1} }$ drawn from an unknown transition kernel $ \smash{ p^*_h(\cdot|s_{t,h}, a_{t,h}) }$. The learner's goal is to maximise $\smash{ V^{\pi_t}(\cdot) }$, its expected reward along each episode $t$ by following $\pi_t$ (see Equation~\eqref{eq: def value function}); or equivalently to minimise the regret $\smash{\Reg_T := \sum_t V_1^{\pi^*}(s_{t,1}) - V_{1}^{\pi_t}(s_{t,1}) }$ against an optimal policy $\pi^*$.

This empirical success has led to extensive theoretical work on provable guarantees for RL with function approximation.
The first line of research considers linear function approximation \citep{jiang2017contextual, jin2020provably, zhou2021provably} where the reward and transition functions are assumed to be linear
in a feature mapping. However, assuming linear transition models in Markov Decision Processes (MDPs) is fundamentally restrictive. Indeed, not all transition probabilities can be accurately captured by linear models \citep[Proposition~1]{hwang2023model}. Moreover, \cite{jin2020provably} show that when the linear model is misspecified, the algorithm may suffer linear regret. 

\cite{hwang2023model} proposed modelling the transition probabilities using a multinomial logistic model, which naturally ensures valid probability distributions over the state space and, unlike linear models, automatically enforces non-negativity and normalisation.
They proposed an algorithm that achieves a regret bound of $ \smash{\tilde O ( \kappa d H^2 \sqrt{T} )} $, where $d$ is the feature dimension, $H$ is the episode length, $T$ is the number of episodes, and $\kappa > 1 $ is a problem-dependent quantity. However, the dependence on $\kappa$ is problematic because it scales polynomially with the number of reachable states, potentially negating the benefits of function approximation in environments with large state spaces. Moreover, it may scale exponentially with respect to the problem parameters.
To resolve this, \cite{li2024provably} designed an improved algorithm that removes $\kappa$ from the dominant regret term, achieving a regret bound of $ \smash{\tilde O(d H^2 \sqrt{T} )} $ alongside constant computational efficiency per episode. They also established a regret lower bound of $ \smash{ \Omega( d H \sqrt{T/\kappa_*}) } $ where $\kappa_*>1$ is a problem-dependent constant measuring the non-linearity of the transition probability at the optimum. Then, \cite{park2024infinite} improved the regret lower bound to $ \smash{ \Omega( d H^{3/2} \sqrt{T}) } $. 
Despite this progress, a fundamental gap remains. Not only is there a $\smash{\sqrt{H}}$ discrepancy between the upper and lower bounds, but the existing $\smash{\tilde O(d H^2 \sqrt{T})}$ algorithms are also strictly worst-case. In particular, they fail to adapt to the structure of the environment, e.g. when transition errors have little effect on rewards.

In parametric bandits, a field closely related to RL, the extension of linear bandits to logistic bandits for binary feedback \citep{faury2020improved} and to multinomial logistic bandits for multiclass feedback \citep{amani2021ucb} has been extensively studied. 
In these parallel settings, the problem-dependent constant $\kappa$ was also initially viewed as a curse worsening the regret bounds.
\cite{faury2020improved} and \cite{amani2021ucb} have shown that the regret is independent of it.
In the binary setting, \cite{abeille2021instance} and \cite{faury2022jointly} provided a minimax optimal regret bound of $\smash{ \tilde O( d \sqrt{T/\kappa_*}) }$ where $\kappa_*>1$ is a problem-dependent constant measuring the non-linearity of the problem at the optimum.
Recently, \cite{boudart2026enjoying} extended these results from binary to multinomial logistic bandits. They generalised the non-linearity constant $\kappa_*$ to the multinomial framework, and designed an efficient algorithm that leverages the problem's non-linearity to yield a minimax optimal regret bound of $\smash{ \tilde{O}(d \sqrt{KT/\kappa_*}) }$ where $K$ is the number of outputs. This raises the central question of our work:

\smallskip
\hspace*{1em}\emph{Can we extend these results to MDPs with MNL transitions, effectively leveraging non-linearity to yield improved regret bounds?}

\paragraph{Main contributions}
\begin{itemize}[topsep=-\parskip, leftmargin=*]

    \item \textbf{A new problem-dependent constant.} We define $\bar\sigma_{T}$ as 
    the normalised average variance of the optimal next-state value function along 
    the learner's trajectory:
    \begin{equation*}
        \bar \sigma_{T}^2 := \frac{1}{TH} \sumT \sumH \frac{1}{H^2} 
        \Var_{p^{*}}\!\left[ V_{h+1}^{*}(s') \mid s_{t,h}, 
        \pi^*_h(s_{t,h}) \right] \,.
    \end{equation*}
    Intuitively, $\bar\sigma_{T}$ captures the idea that transition estimation errors 
    are harmless when all likely next states yield similar optimal downstream rewards. 
    Since $V^*_{h+1} \in [0, H-h]$, one has $\bar\sigma_T^2 \leq 1/4$ in the 
    worst case, but $\bar\sigma_T$ can be significantly smaller for structured MDPs. 
    This definition generalises the non-linearity constant of \citet{boudart2026enjoying} 
    from multinomial logistic bandits to MNL mixture MDPs: unlike in the bandit setting, 
    the ``reward'' of a transition in an MDP is the downstream value $V_{h+1}^*(\cdot)$, 
    which must be taken into account in the definition of $\bar\sigma_T$.

    \item \textbf{A minimax-optimal algorithm.} We introduce 
    Algorithm~\ref{algo:learning routine 2}, whose regret satisfies 
    (Theorem~\ref{theorem:regret bound}):
    \begin{equation*}
        \Reg_T \le \tilde{O}\!\left( d H^{2} \bar\sigma_T \sqrt{T} 
        \log \delta^{-1} \right) \qquad \text{w.p. } 1-5\delta \,,
    \end{equation*}
    answering the above question positively. 
    Since $\bar\sigma_T \leq 1/2$, this recovers the existing 
    $\tilde{O}(dH^2\sqrt{T})$ bound of \citet{li2024provably} in the worst case, 
    and strictly improves upon it whenever the MDP exhibits favourable structure. 
    In particular, for KL-constrained robust MDPs \citep{nilim2005robust, 
    iyengar2005robust}, $\smash{\bar\sigma_T = O(H^{-1})}$ 
    (Proposition~\ref{prop:regret bound KLCR MDP}), reducing the horizon dependence by a factor $H$ compared to the best existing results.
    
    \item \textbf{A matching lower bound.} We prove in 
    Theorem~\ref{theorem:lower bound} that for any algorithm there exists an MNL 
    mixture MDP such that
    \begin{equation*}
        \Reg_T \ge \Omega\!\left(dH^{2} \bar\sigma_T \sqrt{T} \right) \,.
    \end{equation*}
    This lower bound matches our upper bound up to logarithmic factors, establishing 
    the minimax optimality of our algorithm and confirming that the 
    $\bar\sigma_T$ scaling cannot be improved in general. This is the first time 
    that upper and lower bounds align for MNL mixture MDPs, fully characterising 
    the regret complexity of this setting. Notably, the hard instance achieving 
    this lower bound satisfies $\smash{\bar\sigma_T = \Theta(H^{-1/2})}$, which explains 
    why the existing $\smash{\Omega(dH^{3/2}\sqrt{T})}$ lower bound of 
    \citet{park2024infinite} is consistent with (and in fact a special case 
    of) our general result. We provide numerical experiments in 
    Appendix~\ref{appendix:experiments} empirically validating this $\smash{H^{3/2}}$ 
    rate on this instance.

\end{itemize}

\bigskip
We summarise existing algorithms that tackle MNL mixture MDPs in Table~\ref{table:regret comparison}. 

\begin{table}[!ht]
    \centering
    \begin{tabular}{ccc}
        \hline
        & Reference & Regret \\
        \hline
        \multirow{3}{2cm}{\centering Upper bound} 
        & \citep{hwang2023model} & $\tilde O ( \kappa d H^2 \sqrt{T} )$ \\
        & \citep{li2024provably} & $\tilde O( B d H^2 \sqrt{T} )$  \\
        & {\bfseries (ours) - Theorem~\ref{theorem:regret bound} } & $\tilde O\left( B d H^{2}\bar\sigma_T \sqrt{T} \right)$ \\
        \hline
        \multirow{2}{2cm}{\centering Lowerbound} 
        & \citep{park2024infinite} & $\Omega(Bd H^{3/2} \sqrt{T})$ \\
        & {\bfseries (ours) - Theorem~\ref{theorem:lower bound} } & $ \Omega\left(BdH^{2} \bar\sigma_T \sqrt{T} \right) $ \\
        \hline
    \end{tabular}
    \caption[]{Comparison of regret bounds, with respect to $B, d, H, T, \kappa$ and $\bar\sigma_T$. Here $B$ denotes the upper bound on the true parameters ($ \lVert \theta_h^* \rVert_2 \le B $) and $\kappa$ scales exponentially with it ($\kappa = \Theta(e^B)$). For simplicity we omit logarithmic terms and other constants. }
    \label{table:regret comparison}
\end{table}

A central component of our analysis is the self-concordance property of the logistic loss, which lies at the core of recent theoretical developments yielding improved regret bounds in the bandit setting \citep{abeille2021instance, faury2022jointly, boudart2026enjoying}. 
The self-concordance enables us to control the curvature at the true parameter $\theta_h^*$ in terms of the curvature at our estimate $\smash{\hat \theta_{t,h}}$.
To exploit this property our algorithm first performs an exploration routine to construct a confidence set $\Theta_h$ small enough around $\theta_h^*$ to ensure the self-concordance property can be applied with only a constant factor penalty, see Section~\ref{section:exploration routine}.

\section{Problem Formulation}

In this section, we introduce our notations and assumptions, and formally recall the setting of episodic MNL mixture MDPs as introduced by \cite{hwang2023model}.

\paragraph{Framework}
The episodic MNL mixture MDP is formalised as a game of $T\in\N$ episodes of $H\in\N$ rounds between a learner and an environment, see Framework~\ref{framework:E MNL MDP} for a short summary. At each round $h\in\llbracket H \rrbracket$ the learner interacts with an episodic MDP.

\begin{framework}[!ht]
\caption{The episodic Multinomial Logistic (MNL) MDP Framework.}
\label{framework:E MNL MDP}
\textbf{Known: }$\cA, \cS, \{r_h\}_h$ and $\phi$ \\
\textbf{Unknown: }$\{\theta_h\}_h$ \\
\For{Each episode \(t\) in \(1 \dots T\)}{
The learner observes $s_{t,1}$ \\
\For{Each round $h$ in $1 \dots H$}{
    The learner plays action \(a_{t,h}\in\cA\) \\
    The learner gets reward $r_h(s_{t,h}, a_{t,h})$ \\
    The environment generates the next state \( s_{t, h+1} \in \cS_{h,s_{t,h},a_{t,h}} \) such that \( \pr_h[s_{t,h+1}|s_{t,h},a_{t,h}] \propto \exp(\phi(s_{t,h+1}|s_{t,h},a_{t,h})^\top \theta^*_h) \) \\
    The learner observes $s_{t,h+1}$.
}}
\end{framework}

An episodic MDP can be denoted by a tuple \( \cM = ( \cS, \cA, H, \{\pr_h\}_{h=1}^H, \{r_h\}_{h=1}^H) \), where $\cS$ is the state space, $\cA$ is the action space, $H$ is the length of each episode, $\pr_h : \cS \times \cA \times \cS \to [0,1] $ is the transition kernel with $ \pr_h(s'|s,a) $ the probability of transferring to state $s'$ from $s$ while taking action $a$ at stage $h$, and $ r_h: \cS \times \cA \to [0,1] $ is the deterministic reward function at stage $h$. A policy $ \pi = \{ \pi_h \}_{h=1}^H $ is a collection of mappings $ \pi_h : \cS \to \Delta(\cA) $ for every stage $h$. For any policy $\pi$ and $(s,a)\in\cS\times\cA$, we define the action-value function $ Q_h^\pi $ and value function $ V_h^\pi $ as follows:
\begin{equation}\label{eq: def value function}
    Q_h^\pi(s,a) = \E \left[ \sum_{h'=h}^H r_{h'}(s_{h'},a_{h'}) \bigg| s_h=s, a_h=a \right], \qquad V_h^\pi(s)= \E_{a\sim\pi_h(.|s)} [Q_h^\pi(s,a)]
\end{equation}
where the expectation of $Q_h^\pi$ is taken over $\pr$ and $\pi$. The optimal value function $V_h^*$ and action-value function $Q_h^*$ are given by $ V_h^* = \sup_\pi V_h^\pi(s) $ and $ Q_h^*(s,a) = \sup_\pi Q_h^\pi(s,a) $.

The learner interacts with the MDP for $T$ episodes without knowledge of the transition kernel $\pr$. At the beginning of each episode $t$, the learner selects a policy $\smash{\pi_t = \{\pi_{t,h}\}_{h=1}^H}$, and the episode begins at state $s_{t,1}$.
We assume the reward function $ \{r_h\}_h $ to be known to the learner. The goal of the learner is to minimise its regret defined as
\begin{equation*}
    \Reg_T := \sumT V_1^*(s_{t,1}) - \sumT V_1^{\pi_t}(s_{t,1}) 
\end{equation*}
where the optimal policy $\pi^*$ achieves the highest value at every state $s\in\cS$; we denote $\pi^*_h =\argmax_{\pi} V_h^\pi (s) $ for each round $h\in \llbracket H \rrbracket$.
Before formally defining MNL mixture MDPs, we first introduce the notion of \emph{reachable states}. We consider an \emph{inhomogeneous} MDP setting, in which the transition probabilities may vary over the course of an episode.

\begin{restatable}{definition}{}[Reachable States]
For any $ (h, s, a) \in \llbracket H \rrbracket \times \cS \times \cA $ we define the \emph{reachable states} as $ \cS_{h,s,a} := \{ s' \in \cS \;|\; \pr_h(s'|s,a) > 0 \} $. Moreover, we define $ N_{h,s,a} = | \cS_{h,s,a}|  $ and denote by $ U:= \max_{h,s,a} N_{h,s,a} $ the maximum number of reachable states. For any episode $t$ and any stage $h$ we define $ \cS_{t,h} := \cS_{h,s_{t,h}, a_{t,h}} $.
\end{restatable}
We can now formally introduce the definition of MNL mixture MDPs.
\begin{restatable}{definition}{}[MNL mixture MDP]
An MDP instance $ \cM = (\cS, \cA, H, \{\pr_h\}_h, \{r_h\}_h ) $ is called an inhomogeneous, episodic $B$-bounded MNL mixture MDP if there exists a known feature mapping $ \phi(s'|s,a): \cS \times \cA \times \cS \to \R^d $ with $ \lVert \phi(s'|s,a) \rVert_2 \le 1 $ and unknown vectors $ \{\theta_h^*\}_{h=1}^H \in \Theta $ with $ \Theta = \{\theta\in\R^d : \lVert \theta \rVert_2 \le B \} $, such that for all $(s,a,h) \in \cS \times \cA \times \llbracket H \rrbracket $ and $ s' \in \cS_{h,s,a} $ it holds that 
\begin{equation*}
    \pr_h(s'|s,a) = \dfrac{\exp(\phi(s'|s,a)^\top \theta_h^*)}{\sum_{ s'' \in \cS_{h,s,a}} \exp(\phi( s'' |s,a)^\top \theta_h^*) } \,.
\end{equation*}
\end{restatable}

For any $\theta\in\R^d$ we define the induced transition as $ p_{s,a}^{s'}(\theta) = \dfrac{\exp(\phi(s'|s,a)^\top \theta)}{\sum_{ s'' \in \cS_{h,s,a}} \exp(\phi( s'' |s,a)^\top \theta) } $.

\paragraph{Problem-dependent constants $\kappa$ and $ \protect \bar\sigma_T $}
As detailed in the introduction, a key aspect of the binary and multinomial logistic bandits arises from the non-linearity of the softmax function. It is captured both globally and locally at the optimum by the problem-dependent quantities $\kappa$ \citep{faury2020improved, amani2021ucb} and $\kappa_*$ \citep{boudart2026enjoying} respectively. 
\cite{hwang2023model} demonstrated that the global constant $\kappa$ could be extended to the MNL mixture MDP setting. Following the idea of \cite{boudart2026enjoying}, we propose a new definition to quantify the non-linearity at the optimum.
In particular, it depends on the rewards of the optimal policy with the optimal transitions, a feature absent from previous definitions.
Formally $\bar\sigma_{T}$ is defined as follows:
\begin{equation*}
    \bar \sigma_{T}^2 := \dfrac{1}{TH} \sumT \sumH \dfrac{1}{H^2} \Var_{p^{*}} \left[ V_{h+1}^{*} (s') | s_{t,h}, \pi^*_h(s_{t,h}) \right] \,.
\end{equation*}
Intuitively, $\smash{\bar\sigma_{T}}$ represents the averaged normalised variance of the optimal next-state value function, conditioned on the learner taking the optimal action along the trajectory $\smash{(s_{t,h})_{t,h}}$. 
The $\smash{1/H^2}$ scaling normalises the variance so that $\smash{\bar\sigma_{T} \in [0, 1/2]}$, ensuring that our measure of non-linearity lies in a constant range independent of the horizon $H$. 
This definition provides a direct bridge to the bandit literature, when $H=1$ and there is no state space $\cS$. The problem reduces to an MNL bandit where $\smash{V_{h+1}^*(s')}$ simplifies to the immediate reward of outcome $s'$. In this regime $\smash{\bar\sigma_T^2}$ exactly recovers the variance of the optimal arm's reward, which corresponds to the inverse of the bandit non-linearity constant $\kappa_*^{-1}$ used to achieve the $\smash{\tilde{O}(d\sqrt{T/\kappa_*})}$ bounds in \cite{abeille2021instance} and \cite{boudart2026enjoying}.
Variance-based quantities of this kind have a long history in online learning, \cite{cesa2007improved,gaillard2014second} showed that the regret of prediction with expert advice can be bounded by the cumulative variance along the trajectory.
Because our regret bound scales with the square root of this average variance, the learning problem becomes significantly easier (i.e., $\bar \sigma_{T}$ is small) in two highly favourable, yet common, scenarios:
\begin{itemize}[nosep, leftmargin=*]
    \item[-] Deterministic optimal dynamics: if taking the optimal action almost certainly leads to a specific next state. The variance drops to zero, reflecting that the learner does not need to learn the transition probabilities with high precision to act optimally.
    \item[-] Downstream value equivalence: Even if transitions are highly stochastic, if the probability mass under the optimal policy is distributed over states with nearly identical downstream returns ($\smash{V_{h+1}^{\pi^*}}$), the variance is inherently small. 
\end{itemize}
Proposition~\ref{prop:regret bound KLCR MDP} exhibits a classical family of MDPs 
for which $\smash{\bar\sigma_{T} \leq \log U / (H\eta)}$, where $\eta > 0$ is a 
robustness parameter of the MDP. Furthermore, if $\bar\sigma_T$ were evaluated 
under the state distribution induced by the optimal policy $\pi^*$ rather than 
along the learner's trajectory $(s_{t,h})_{t,h}$, the bound would tighten to 
$\smash{\E_{\pi^*}[\bar\sigma_T^2] \leq \frac{1}{4H}}$ for any MDP 
(Proposition~\ref{prop:bound bar sigma optimal trajectory}).

\paragraph{Assumptions}
We summarise our assumptions below. Most are standard in the literature \citep{hwang2023model, li2024provably}, while the last assumption is new to this work.
\begin{itemize}[nosep, topsep=-\parskip, leftmargin=*]
    \item The norm of the features is bounded by 1: for all $\smash{s, a, s' \in \cS \times \cA \times \cS, \lVert \phi(s'|s,a) \rVert_2 \le 1 }$.
    \item For each round $h\in \llbracket H \rrbracket$, the norm of the parameter $\theta_h^* \in \R^d$ is bounded by $B\ge 1$: $\lVert \theta_h^* \rVert_2 \le B$. The bound $B$ is known.
    \item For each round $h\in \llbracket H \rrbracket$, the reward function $r_h: \cS \times \cA \to [0,1]$ is known.
    \item There exists $\kappa >1$ such that for all $ h,s,a \in \llbracket H \rrbracket \times \cS \times \cA $ and $ s',s''\in \cS_{h,s,a} $ and for all $\theta$ such that $\lVert \theta \rVert_2 \le B $, we assume
    \begin{equation}\label{eq:definition kappa}
        p_{s,a}^{s'}(\theta) p_{s,a}^{s''}(\theta) \ge \dfrac{1}{\kappa} \,.
    \end{equation}
    \item There exists $\rho >0 $ such that for every stage $h\in\llbracket H \rrbracket$, we have 
    \begin{equation}\label{eq:definition rho}
        \max_{s,a\in\cS\times\cA} \max_{s'\in\cS_{h,s,a}} \lVert \phi(s'|s,a) \rVert_2 \ge \rho > 0 \,.
    \end{equation}
\end{itemize}
Note that the assumption $ \max_{s,a,s'} \lVert \phi(s'|s,a) \rVert_2 \le 1 $ is made without loss of generality. Indeed, the norm of the features can be transferred to the norm of $\theta_h^*$. The assumption on $\kappa$ allows us to lower-bound the eigenvalues of the Hessian of the logistic loss.

\paragraph{Additional Notations}
Given a compact set $\Theta$, we define its diameter under an action space $\cA$ and a state space $\cS$ as
\begin{equation*}
    \diam_{\cA, \cS, h}(\Theta) := \max_{a\in\cA} \max_{s\in\cS} \max_{s'\in\cS_{h,s,a}} \max_{\theta_1, \theta_2 \in \Theta}  | (\theta_1 - \theta_2)^\top \phi(s'|s,a) | \,.
\end{equation*}
The notation $\lesssim$ indicates an inequality up to a universal constant.
We denote by $\ell_{t,h}$ the logistic loss at stage $(t,h)$, defined for all $\theta\in\R^d$ by $\smash{\ell_{t,h}(\theta) := - \log p_{s_{t,h},a_{t,h}}^{s_{t,h+1}} (\theta)}$ 
and by $\smash{H_{t,h}(\theta) := \nabla^2 \ell_{t,h}(\theta) }$ its Hessian.

\section{Algorithm and Regret Analysis}

In this section, we introduce \textit{LIVAROT} (Logistic Inference with Variance-Aware Robust Optimistic Targets), see Algorithm~\ref{algo:learning routine 2}, and derive a bound on its regret. 
\textit{LIVAROT} follows the explore-and-learn paradigm. Following the idea of \cite{abeille2021instance} for binary logistic bandits, the first exploration phase aims to design sufficiently small confidence sets $\Theta_h$ around $\smash{\theta_h^*}$ for every $h\ge 1$.
During the learning phase, \textit{LIVAROT} relies on an optimistic state-value function to address the exploration-exploitation trade-off. To obtain a computationally efficient algorithm, the parameter updates are performed via Online Mirror Descent (OMD) using quadratic approximations of the logistic loss.

\subsection{Exploration Routine}\label{section:exploration routine}
We first introduce our exploration routine (see Algorithm~\ref{algo:exploration routine}) and discuss the challenges associated with it. This exploration routine is then used as an initialisation phase in our main algorithm (see Algorithm~\ref{algo:learning routine 2}).

\begin{algorithm}[!ht]
\caption{\textsc{Exploration\_Routine}}
\label{algo:exploration routine}
\KwIn{Length of the procedure \(\tau\), regularisation parameter \(\lambda_0\)}
{\bfseries{Init:} \(\tilde U_{1,h} = I_d \quad \forall h \in \llbracket H \rrbracket \)} \\
\For{each episode \(k\) in \(1 \dots \tau\)}{
    \For{each stage $h$ in $1 \dots H$}{
        Choose $\smash{a_{t,h}, \tilde s_{t,h}, \tilde s'_{t,h} = \argmax_{a\in\cA} \argmax_{s\in\cS} \argmax_{s'\in\cS_{h,s,a}} \lVert \phi(s'|s,a) \rVert^2_{\tilde U_{t,h}^{-1}}  }$ \\
        Observe $ s_{t,h+1} \sim \pr_h(. | s_{t,h} , a_{t,h}) $ \\
        Get reward $ r_h(s_{t,h}, a_{t,h}) $
    }
    Update $\smash{ \tilde U_{t+1,h} = \tilde U_{t,h} + \tfrac{1}{\kappa} \rho \phi(\tilde s'_{t,h} | \tilde s_{t,h} , a_{t,h} ) \phi(\tilde s'_{t,h} | \tilde s_{t,h} , a_{t,h} )^\top  \quad \forall h \in \llbracket H \rrbracket  }$ 
}
Compute for all $h\in\llbracket H \rrbracket$: 
\begin{multline*}
    \hat \theta_{\tau+1, h} = \argmin_{\theta \in \R^d} \Bigg[  \mathcal{\tilde L}_{\tau+1,h}(\theta) := \sum_{t=1}^\tau -  \log p_{t,h}^{s_{t,h+1}}(\theta) + \dfrac{\lambda_0+1}{2} \lVert \theta \rVert_2^2 \\
    + \sum_{t=1}^\tau - \lVert \phi(\tilde s'_{t,h} | \tilde s_{t,h} , a_{t,h}) \rVert_{\tilde U_{t+1,h}^{-1}} \log p_{\tilde s_{t,h}, a_{t,h}}^{\tilde s_{t,h}'}(\theta)  \Bigg]
\end{multline*} \\
Compute for all $h\in\llbracket H \rrbracket$:
$$\!\!\! A_{\tau+1,h} = \tfrac{1}{\kappa} \sum_{t=1}^\tau \sum_{s'\in\cS_{t,h}} \phi(s'| s_{t,h}, a_{t,h}) \phi(s'| s_{t,h}, a_{t,h})^\top + \tfrac{\rho}{\kappa} \sum_{t=1}^\tau \phi(\tilde s'_{t,h} | \tilde s_{t,h} , a_{t,h} ) \phi(\tilde s'_{t,h} | \tilde s_{t,h} , a_{t,h} )^\top + I_d $$ \\
$ \Theta_h := \left\{ \theta \in \R^d : \lVert \theta - \hat \theta_{\tau+1, h} \rVert^2_{A_{\tau+1, h}} \le \beta^0_{\tau+1}(\delta) \right\} \quad \forall h \in \llbracket H \rrbracket $ see Equation~\eqref{eq:def beta_t}
\end{algorithm}

The goal of our exploration routine is to produce confidence sets $\Theta_h$ such that $\theta_h^*\in\Theta_h$ for all $h\ge 1$ with high probability and $\diam_{\cA, \cS, h}(\Theta_h) \le 1$. This enables us to leverage the self-concordance property \citep[Proposition~8]{sun2019generalized} of the logistic loss at constant cost. For all stages $t,h\ge1$, for all $s,a\in\cS\times\cA$ and $ s'\in\cS_{h,s,a} $ we have w.h.p.:
\begin{equation*}
    H_{t,h}(\theta_1) \preccurlyeq 
    \exp\left(3\sqrt{2} \diam_{\cA,\cS,h}(\Theta_h) \right) H_{t,h}(\theta_2) 
    \preccurlyeq e H_{t,h}(\theta_2) \qquad \forall \theta_1, \theta_2 \in \Theta_h 
\end{equation*}
where we recall that $H_{t,h}$ is the Hessian of the logistic loss $\ell_{t,h}$.
The following lemma shows that such sets $\smash{(\Theta_h)_{h=1}^H}$ can be obtained with a reasonably small exploration length $\tau$. While the theoretical value of $\tau$ may appear large, it reflects a worst-case bound that only affects the lower-order terms in the regret, and substantially smaller values suffice in practice.
The proof is deferred to Appendix~\ref{appendix:proof lemma constant diameter}.

\begin{restatable}{lmm}{LemmaConstantDiameter}\label{lemma:constant diameter}
Let $\delta\in(0,1], \lambda_0 = \sqrt{d\log(\tau H/\delta)} $, 
\begin{align}
    \beta^0_{\tau+1}(\delta) &:= (1+3\sqrt{2}) \left[(B+3) \sqrt{d \log( (\tau+1) H/\delta) } + (\tau+2)^{1/4} \left(\tfrac{\kappa}{\rho}\right)^{3/2} d\log\left( 1 + \tfrac{\tau+1}{d} \right) \right] \label{eq:def beta_t} \\
    \text{and}\qquad
    \tau &:= 4^5 \left(\tfrac{\kappa}{\rho}\right)^{8} (1+3\sqrt{2})^4 d^6 (B+3)^2 \log\left(\tfrac{TH}{\delta}\right) \left[\log\left(1+\tfrac{T}{d}\right)\right]^6 \,. \nonumber
\end{align}
Then the sets $ (\Theta_h)_{h=1}^H $ returned by Algorithm~\ref{algo:exploration routine} satisfy with probability $1-\delta$
\begin{equation*}
    \theta_h^* \in \Theta_h \quad \text{and} \quad \diam_{\cA, \cS, h}(\Theta_h) \le \dfrac{1}{3\sqrt{2}} \quad \text{for all } h\in\llbracket H \rrbracket \,.
\end{equation*}
\end{restatable}
As $\cS, \cA$ and $B$ are known to the learner, $\kappa$ and $\rho$ can be computed by the learner (defined in Equations~\eqref{eq:definition kappa} and~\eqref{eq:definition rho}). Alternatively, the upper bound $\kappa\le U^2 \exp(4B)$ can also be employed \citep[Claim~1]{li2024provably}, which is the approach we take for our experiments in Appendix~\ref{appendix:experiments}.

\subsection{Learning Phase}
We introduce the core of our algorithm, which leverages the exploration routine (see Algorithm~\ref{algo:learning routine 2}). To select an action and address the exploration-exploitation trade-off, we use the Optimism in the Face of Uncertainty (OFU) paradigm. At each episode $t$ and each stage $h$, the learner selects an action according to the rule
\begin{equation*}
    a_{t,h} \in \argmax_{a\in\cA} \tilde Q_{t,h} (s_{t,h},a)
\end{equation*}
where $s_{t,h}$ is the current state and $ \tilde Q_{t,h} $ is an optimistic state-action value function that upper bounds the true value $ \smash{Q_{h}^{*} }$. A classical approach for defining $ \smash{\tilde Q_{t,h} }$ is to build a confidence set $\cC_{t,h}(\delta)$ at each stage $(t,h)$ containing $\theta_h^*$ and define recursively:
\begin{equation}
    \tilde Q_{t,h}(s,a) := r_h(s,a) + \max_{\theta\in\cC_{t,h}(\delta)} \sum_{s'\in\cS_{h,s,a}} p_{s,a}^{s'}(\theta) \tilde V_{t,h+1}(s')
    \quad \text{and} \quad 
    \tilde V_{t,h}(s) := \argmax_{a\in\cA} \tilde Q_{t,h}(s,a)
    \,. \label{eq:definition optimistic reward}
\end{equation}
This definition results in a non-concave maximisation problem, which is typically computationally challenging to solve. 
A common alternative to the joint maximisation is to replace the inner $\smash{\max_{\theta\in\cC_{t,h}(\delta)} }$ with an additive exploration bonus \cite{hwang2023model,li2024provably}. However, constructing such a bonus requires bounding $ \smash{\lVert \tilde V_{h+1} \rVert_{\infty} \le H }$ and thus discards the variance information that drives our problem-dependent constant $\bar \sigma_T$. We therefore retain the joint maximisation, which preserves variance-adaptivity at the cost of a non-concave optimisation step solved approximately in practice. In our simulations, we tackle this problem by employing the Frank-Wolfe method to efficiently converge to a local maximum, see Appendix~\ref{appendix:experiments}.

We first compute an estimate $\smash{\htheta_{t,h}}$ of $\smash{\theta_h^*}$ associated to the confidence set
\begin{equation*}
    \cC_{t,h}(\delta) = \{ \theta \in \R^d \; | \; \lVert \theta - \htheta_{t,h} \rVert_{\cH_{t,h}} \le \beta_t(\delta) \}
\end{equation*}
where $ \smash{\cH_{t,h} := \sum_{s=1}^{t-1} H_{s,h}(\htheta_{s+1,h}) + \lambda I_d }$. Our estimate $\htheta_{t,h} $ is computed via an Online Mirror Descent procedure following the second-order approximation of the logistic loss introduced by \cite{zhang2024online} in the bandit setting.
Formally, our estimate $\smash{\htheta_{t+1,h}}$ is the solution of the following optimisation problem
\begin{equation*}
    \htheta_{t+1,h} = \argmin_{\theta\in\Theta_h} \langle \nabla \ell_{t,h} (\htheta_{t,h}), \theta - \htheta_{t,h} \rangle + \tfrac{1}{2} \lVert \theta - \htheta_{t,h} \rVert_{\tilde \cH_{t,h}}^2 
\end{equation*}
where $ \tilde \cH_{t,h} := \cH_{t,h} + H_{t,h}(\htheta_{t,h}) $.
The following lemma from \cite{lee2025improved} shows that $\theta_h^* \in \cC_{t,h}(\delta)$ with high probability.
\begin{restatable}{lmm}[\citep[Theorem~4.2]{lee2025improved}]\label{lemma:confidence set learning phase}
Let $\delta\in]0,1]$. Set $ \lambda = 144 d $. Let us assume Lemma~\ref{lemma:constant diameter} holds. Define $\smash{\cC_{t,h}(\delta) := \{ \theta \in \R^d | \lVert \theta - \htheta_{t,h} \rVert_{\cH_{t,h}} \le \beta_t(\delta) \} }$
where $\beta_t(\delta) = (\tfrac{2}{3} d \log(t/\delta))^{1/2} + 24 B \sqrt{d} $. Then we have with probability $1-\delta$,
$\theta_h^* \in \cC_{t,h}(\delta)$ for all $t,h\ge 1$.
\end{restatable}

\begin{algorithm}[!ht]
\caption{\textit{LIVAROT: Logistic Inference with Variance-Aware Robust Optimistic Targets}}
\label{algo:learning routine 2}
\KwIn{Exploration length \(\tau\), regularisation parameters $\lambda_0$ and \(\lambda\)}
{\bfseries{Init:} \(\cH_{\tau+1,h} = \lambda I_d, \htheta_{\tau+1,h} = 0_d \quad \forall h \in \llbracket H \rrbracket \)} \\
Run $\Theta \gets$ \textsc{Exploration\_Routine}($\tau, \lambda_0$) \\
\For{each episode \(t\) in \(\tau+1 \dots T\)}{
    Compute $ (\tilde Q_{t,h}(.,.))_{h=1}^H $ in a backward way as in Equation~\eqref{eq:definition optimistic reward} \\
    \For{each stage $h$ in $1 \dots H$}{
        Observe state $s_{t,h}$, select action $\smash{ a_{t,h} \in \argmax_{a\in\cA} \tilde Q_{t,h} (s_{t,h}, a) }$ \\
        Update $ \smash{\tilde \cH_{t,h} = \cH_{t,h} +  H_{t,h}(\htheta_{t,h}) }$ \\
        Compute $ \htheta_{t+1,h} = \argmin_{\theta\in\Theta_h} \langle \nabla \ell_{t,h} (\htheta_{t,h}), \theta - \htheta_{t,h} \rangle + \tfrac{1}{2} \lVert \theta - \htheta_{t,h} \rVert_{\tilde \cH_{t,h}}^2 $ \\
        Update $ \smash{\cH_{t+1,h} = \cH_{t,h} + H_{t,h}(\htheta_{t+1,h}) }$
    }
}
\end{algorithm}

\subsection{Regret Analysis}

We now introduce our regret bound for \textit{LIVAROT}, see Algorithm~\ref{algo:learning routine 2}. The proof is deferred to Appendix~\ref{appendix:proof regret bound}.

\begin{restatable}{thm}{TheoremRegretBound}[Regret Bound]\label{theorem:regret bound}
Let $\delta \in ]0, 1]$. Set hyperparameters $\tau, \lambda_0$ as in Lemma~\ref{lemma:constant diameter} and $\lambda=144d$. Then, with probability at least $1-5\delta$, the regret of Algorithm~\ref{algo:learning routine 2} satisfies:
\begin{equation*}
    \Reg_T \le \tilde{O} \left( B d H^{2} \bar \sigma_T \sqrt{ T } \log \delta^{-1} \right) \,,
\end{equation*}
where $\tilde{O}(\cdot)$ hides logarithmic factors in $T$ and lower-order terms. 
\mainonly{We present the dominant term of the bound; the full version is provided in Appendix~\ref{appendix:proof regret bound}.}
\appendixonly{More precisely, the exact upper bound is given by:
\begin{align*}
     \Reg_T \le 
     & 4 \sqrt{2e} B d H^2 \bar\sigma_T \sqrt{T} \log(1+T) \sqrt{\log(T/\delta)} \\
     +& H^{3/2} \sqrt{2T\log(2/\delta)} \\
     +& 8 \sqrt{e \nu_{\max}} B d H^{5/4} T^{1/4} \log(1+T) \log(T/\delta) \sqrt{\log(2/\delta)} \\
     +&8e^2 B d H^{9/4} T^{1/4} \sqrt{\log(1+T)\log(T/\delta)\log(2/\delta)} \\
     +&32e B^2 \nu_{\max} d^2 H \log(1+T) \log(T/\delta) \\
     +& 32e^2 B^2 d^2 H^3 \log^2(1+T) \log(T/\delta) \\
     +& 40\kappa B^2 d^2 H^2 \log( 1 + T ) \log(T/\delta) \\
    +& 4^5 (1+3\sqrt{2})^4 \dfrac{\kappa^8}{\rho^8} d^6 H (B+3)^2 \log(TH/\delta) \log^6 T \,.
\end{align*}
}
\end{restatable}

The regret bound, $ \smash{ dH^{2} \bar\sigma_T \sqrt{T} } $, highlights the most crucial aspect of our analysis.
It shows that the non-linearity of the problem positively influences the regret bound in the long-term regime, contrasting with the previous results from the literature \citep{hwang2023model, li2024provably}.
Since $\smash{\bar\sigma_{T} \le 1/2} $, this term is never worse than its counterpart in prior work and becomes strictly smaller when the problem exhibits favourable non-linearity.
In the worst-case scenario ($\smash{\bar \sigma_{T} = 1/2}$ for all $t,h\ge 1$) we recover the regret bound $ \smash{ \tilde O( d H^2 \sqrt{T}) } $ of \cite{li2024provably}.
However in some cases $\smash{\bar\sigma_{T}}$ can be as low as $\smash{H^{-1}}$, see Proposition~\ref{prop:regret bound KLCR MDP}, leading to a regret bound of $\smash{\tilde O(dH \sqrt{T})}$.
In Appendix~\ref{appendix:experiments}, we provide numerical experiments to validate empirically the $\smash{H^{3/2}}$ rate using the challenging MDP family introduced by \cite{park2024infinite} to establish their $\smash{\Omega(dH^{3/2}\sqrt{T})}$ lower bound.

Together, these regimes illustrate how our variance-aware framework bridges the theoretical gap for MNL mixture MDPs. By successfully capturing the worst-case upper bound \citep{li2024provably}, the fundamental lower bound for hard instances \citep{park2024infinite}, and highly favourable non-linear environments, our approach provides a tight and unified characterisation of the problem's underlying complexity.

The following lower bound shows that for any horizon $H$, number of episodes $T$ and any parameter dimension $d$ there exists an MDP instance for which the learner incurs a regret proportional to $\smash{\bar \sigma_T}$.
The proof is deferred to Appendix~\ref{appendix:proof lower bound}.

\begin{restatable}{thm}{TheoremLowerBound}\label{theorem:lower bound}
Suppose that $d\ge 2, H\ge 3, T \ge \{ (d-1)^2 H/2, (d-1)^2 H^3 / 32 \} $. Then for any algorithm, there exists an MNL mixture MDP such that 
\begin{equation*}
    \E\left[\Reg_T \right] \ge \Omega \left( B d H^{2} \bar \sigma_T \sqrt{T} \right) \,.
\end{equation*}
\end{restatable}
In our proof the bound on $\bar \sigma_{T}$ is computed independently of the trajectory $(s_{t,h})_{t,h}$, and therefore of the learner's algorithm.
This lower bound demonstrates that the preceding regret upper bound is minimax optimal up to a $\log$ term and that our choice of the problem-dependent constant $\bar\sigma_{T}$ is optimal.
To prove our lower bound, we consider a slightly modified MDP instance introduced by \cite{park2024infinite} and compute the resulting non-linearity constants.

\paragraph{KL-constrained robust MDPs.}
Robust MDPs were introduced by \citet{nilim2005robust} and \citet{iyengar2005robust} 
to address the sensitivity of dynamic programming solutions to misspecification of the 
transition kernel. Rather than committing to a single model, the learner posits an 
\emph{uncertainty set} $\cP_h(s,a)$ around a nominal kernel $p_{0,h}(\cdot|s,a)$ and 
optimises the worst-case value
\begin{equation*}
    \min_{p \in \cP_h(s,a)} \E_{s' \sim p}\bigl[ V_{h+1}^*(s') \bigr] \,,
\end{equation*}
which, under the standard $(s,a)$-rectangularity assumption \citep{iyengar2005robust}, 
yields a tractable robust Bellman operator. Among the divergences used to define 
$\cP_h(s,a)$ (total variation, $\chi^2$, Wasserstein, etc.) the Kullback-Leibler 
ball $\cP_h(s,a) = \{ p : \mathrm{KL}(p \| p_{0,h}(\cdot|s,a)) \le \epsilon \}$ has 
received particular attention \citep{iyengar2005robust, filippi2010optimism, 
panaganti2022sample} because Lagrangian duality yields a closed-form worst-case kernel
\begin{equation*}
    p^*_h(s' | s, a) 
    \propto p_{0,h}(s'|s, a) \exp\!\bigl(-\eta V_{h+1}^*(s')\bigr) \,,
\end{equation*}
where $\eta>0$ is the dual variable associated with the KL radius $\epsilon$.

We now show that KL-constrained robust MDPs fit within our MNL framework. Suppose 
there exist $d \geq 1$, a known feature map $\psi:\cS\to\R^{d-1}$, and unknown 
parameters $\alpha_{h+1} \in \R^{d-1}$ such that 
$V_{h+1}^*(s') = \alpha_{h+1}^\top \psi(s')$ for all $s' \in \cS$ and $h \in [H-1]$. 
This representation is always achievable by setting $d = |\cS|+1$, 
$\psi(s') = e_{s'}$ (the canonical basis vector indexed by $s'$), and 
$\alpha_{h+1} = (V_{h+1}^*(s'))_{s'\in\cS}$, though in practice a lower-dimensional 
$\psi$ may be available. Define
\begin{equation*}
    \phi(s'|s,a) := 
    \begin{pmatrix} \log p_{0,h}(s'|s,a) \\ \psi(s') \end{pmatrix} \in \R^{d},
    \qquad
    \theta_h^* := 
    \begin{pmatrix} 1 \\ -\eta \alpha_{h+1} \end{pmatrix} \in \R^{d} \,.
\end{equation*}
The induced MNL transition 
$p_h^*(s'|s,a) \propto \exp(\phi(s'|s,a)^\top\theta_h^*)$ then recovers the 
KL-constrained worst-case kernel 
$p_h^*(s'|s,a)\propto p_{0,h}(s'|s,a)\exp(-\eta V_{h+1}^{*}(s'))$, as required.

Under some weak additional assumptions on the feature map, the upper bound $B$ on the norms of $\|\theta_{h+1}\|$ can be controlled as follows. Assuming 
$\smash{\lVert \log p_{0,h}(\cdot|s,a) \rVert_\infty \leq 1}$ and 
$\smash{\lVert \psi(s') \rVert_2 \leq 1}$ yields $\smash{\lVert \phi(s'|s,a) \rVert_2 \leq \sqrt{2}}$; 
the factor $\sqrt{2}$ can be absorbed into $\theta_h^*$ by rescaling without loss of 
generality. To control $\lVert \theta_h^* \rVert_2$, we further assume the features are 
orthonormal under some reference distribution $\mu$, i.e.\ 
$\smash{\E_{\mu}[\psi(s')\psi(s')^\top] = I_{d-1}}$. Under this assumption 
$\smash{\lVert \alpha_{h+1} \rVert_2 = \lVert \alpha_{h+1} \rVert_{L^2(\mu)} 
\leq \lVert V_{h+1}^* \rVert_\infty \leq H}$, so that
\begin{equation*}
    B := \max_h \lVert \theta_h^* \rVert_2 
    = \max_h \sqrt{1+\eta^2\lVert\alpha_{h+1}\rVert_2^2} 
    \leq \sqrt{1 + \eta^2 H^2} \,.
\end{equation*}
We now upper bound the problem-dependent constant $\bar\sigma_{T}$ and the regret for 
KL-constrained robust MDPs. The proof is deferred to 
Appendix~\ref{appendix:KLCR MDP}.

\begin{restatable}{prop}{PropositionRegretKLCRMDP}
\label{prop:regret bound KLCR MDP}
Let $\cM$ be a KL-constrained robust MDP as defined above. Then the problem-dependent 
constant $\bar\sigma_{T}$ satisfies
\begin{equation*}
    \bar\sigma_{T}^2 
    \le \min\!\left( \frac{1}{4},\, 
    \frac{\log^2 U + 2 \log U + 2}{H^2 \eta^2} \right) \,.
\end{equation*}
Let $\delta \in (0,1]$ and set the hyperparameters $\tau, \lambda_0$ and $\lambda$ as 
in Theorem~\ref{theorem:regret bound}. Then, with probability at least $1-5\delta$, 
the regret of Algorithm~\ref{algo:learning routine 2} satisfies
\begin{equation*}
    \Reg_T \le \tilde{O}\!\left( 
    dH^2 \min\!\left\{\eta H,\, \log U\right\} \sqrt{T}\, \log \delta^{-1} 
    \right) \,.
\end{equation*}
\end{restatable}
In contrast, the best existing result for KL-constrained robust MDPs with value 
function approximation \citep{li2024provably} yields a regret of order 
$\smash{\tilde{O}\big(\eta (\log U)\, dH^3 \sqrt{T} \log(1/\delta)\big)}$. Our bound 
improves upon this in two ways: it gains a factor of $H$ in the horizon dependence, 
and it saturates as $\eta \to \infty$ whereas the existing bound becomes vacuous in such regime.

\section{Conclusion}
In this paper, we resolve the theoretical gap in MNL mixture MDPs by showing that the non-linearity of the softmax can be exploited to yield tighter regret bounds.
We introduce $\bar \sigma_T$, a problem-dependent constant capturing the variance of the optimal downstream returns and provide a minimax-optimal algorithm that adapts to the environment's structure, proving that the previous strictly worst-case bounds can be bypassed.
While our variance-aware framework provides a tight characterisation of episodic MNL MDPs, it also reveals several limitations that suggest promising directions for future research.
First, our definition of $\bar\sigma_T$ is evaluated along the learner's trajectory $\smash{(s_{t,h})_{t,h}}$. An important open question is whether this bound can be defined solely based on the environment, for example by evaluating the variance exclusively over the state distribution induced by the optimal policy, which would improve the dependence on $H$ (see Proposition~\ref{prop:bound bar sigma optimal trajectory}). 
Second, to leverage the self-concordance property of the logistic loss we rely on an initial exploration phase. However, as suggested by our numerical experiments, this burn-in period appears unnecessary in practice. Designing a single-phase algorithm that bypasses the forced exploration remains an interesting challenge.
Moreover, our framework currently assumes a stationary environment where the transition parameters $\theta_h^*$ remain fixed across episodes. An interesting direction for future work is to accommodate non-stationary parameters $\theta^*_{t,h}$, possibly adversarial, which would require extending our analysis to track drifting transition dynamics and bound the dynamic regret.
Another direction would be to replace the non-concave maximisation in \eqref{eq:definition optimistic reward} by a Thompson sampling scheme, which would sidestep the inner optimization, or to design a suitable additive exploration bonus that retains the dependence on $\bar \sigma_T$.
Finally, this work focuses exclusively on the finite-horizon episodic setting. Extending it to infinite-horizon MDPs introduces new challenges, particularly regarding how the non-linearity of the transition model influences the long-term reachability of states and the learner's ability to recover from sub-optimal states.

\paragraph{Acknowledgements.}
A.R. acknowledges the support of the French government under management of Agence Nationale de la Recherche as part of the “Investissements d’avenir” program, reference ANR-19-P3IA-0001 (PRAIRIE 3IA Institute) and the support of the European Research Council (grant REAL 947908).
This work was supported by funding from the French government, managed by the National Research Agency (ANR), under the France 2030 program, reference ANR-23-IACL-0006.

\newpage
\bibliography{bibfile}

\newpage
\appendix
\inappendixtrue
\begin{center}
    \huge
    APPENDIX
\end{center}

This appendix is organised as follows:
\begin{itemize}[nosep, leftmargin=*]
    \item[-] Appendix~\ref{appendix:notations}: Notations
    \item[-] Appendix~\ref{appendix:analysis algorithm}: Analysis of Algorithm~\ref{algo:learning routine 2}
    \item[-] Appendix~\ref{appendix:proof lower bound}: Proof of Theorem~\ref{theorem:lower bound} -- Lower bound
    \item[-] Appendix~\ref{appendix:KLCR MDP}: KL-Constrained Robust MDPs
    \item[-] Appendix~\ref{appendix:bound bar sigma}: Bound on $\bar \sigma_T$ under an optimal policy $\pi^*$
    \item[-] Appendix~\ref{appendix:experiments}: Experiments
    \item[-] Appendix~\ref{appendix:auxiliary results}: Auxiliary Result
\end{itemize}

\section{Notations}\label{appendix:notations}

We detail below useful notations used throughout the appendix.
\begin{itemize}[leftmargin=1cm]
    \setlength\itemsep{0.3cm}
    \item[-] $\llbracket T \rrbracket := \{1, 2, \dots, T \} \quad, \forall T \in \N^*$
    \item[-] Diameter of the set $\Theta$:
    \begin{equation*}
        \diam_{\cA, \cS, h}(\Theta) := \max_{a\in\cA} \max_{s\in\cS} \max_{s'\in\cS_{h,s,a}} \max_{\theta_1, \theta_2 \in \Theta}  | (\theta_1 - \theta_2)^\top \phi(s'|s,a) |
    \end{equation*}
    \item[-] Reachable states at stage $(t,h)$: $\cS_{t,h} := \cS_{h,s_{t,h}, a_{t,h}} $
    \item[-] Logistic loss at stage $(t,h)$, for all $\theta\in\R^d$:
    \begin{equation*}
        \ell_{t,h}(\theta) := \sum_{s'\in\cS_{t,h}} -\mathbf{1}[s'=s_{t,h+1}] \log p_{s_{t,h},a_{t,h}}^{s'} (\theta)
    \end{equation*}
    \item[-] Hessian of the logistic loss \cite[Appendix~A]{li2024provably}:
    \begin{equation*}
    H_{t,h}(\theta) := \sum_{s'\in\cS_{t,h}} p_{s_{t,h},a_{t,h}}^{s'}(\theta) \phi_{s_{t,h},a_{t,h}}^{s'} (\phi_{s_{t,h},a_{t,h}}^{s'})^\top 
    - \sum_{s''\in\cS_{t,h}} p_{s_{t,h},a_{t,h}}^{s'}(\theta) p_{s,a}^{s''}(\theta) \phi_{s_{t,h},a_{t,h}}^{s'}(\phi_{s_{t,h},a_{t,h}}^{s''})^\top \,.
\end{equation*}
\end{itemize}

\section{Analysis of Algorithm~\ref{algo:learning routine 2}}\label{appendix:analysis algorithm}

\subsection{Exploration Routine}

Let us first introduce some useful notations for the analysis.
We define the gradient of $\mathcal{\tilde L}_{t,h}(\theta) $ by
\begin{equation*}
    \cG_{t,h}(\theta) := \sum_{i=1}^{t-1} \sum_{s'\in\cS_{i,h}} (p_{i,h}^{s'}(\theta) - y_{i,h}^{s'}) \phi_{i,h}^{s'} + (\lambda_0 + 1) \theta + \sum_{i=1}^{t-1} \lVert \phi(\tilde s'_{t,h} | \tilde s_{t,h} , a_{t,h}) \rVert_{\tilde U_{t+1,h}^{-1}} (p_{\tilde s_{i,h}, a_{i,h}}^{\tilde s'_{i,h}} (\theta) -1) \phi(\tilde s'_{i,h} | \tilde s_{i,h} , a_{i,h})
\end{equation*}
and its Hessian by
\begin{equation*}
    W_{t,h}(\theta) := \cH_{t,h}^0(\theta) + U_{t,h}(\theta)
\end{equation*}
where
\begin{equation*}
    \cH_{t,h}^0(\theta) := \sum_{i=1}^{t-1} \sum_{s'\in\cS_{i,h}} p_{i,h}^{s'}(\theta) \phi_{i,h}^{s'} (\phi_{i,h}^{s'})^\top - \sum_{i=1}^{t-1} \sum_{s'\in\cS_{i,h}} \sum_{s''\in\cS_{i,h}} p_{i,h}^{s'}(\theta) p_{i,h}^{s''}(\theta) \phi_{i,h}^{s'} (\phi_{i,h}^{s''})^\top + \lambda_0 I_d
\end{equation*}
and 
\begin{equation*}
    U_{t,h}(\theta) := \sum_{i=1}^{t-1} \lVert \phi(\tilde s'_{i,h} | \tilde s_{i,h} , a_{i,h}) \rVert_{\tilde U_{i+1,h}^{-1}} p_{\tilde s_{i,h}, a_{i,h} }^{\tilde s'_{i,h} }(\theta) (1-p_{\tilde s_{i,h}, a_{i,h} }^{\tilde s'_{i,h} }(\theta)) \phi(\tilde s'_{i,h} | \tilde s_{i,h}, a_{i,h} ) \phi(\tilde s'_{i,h} | \tilde s_{i,h}, a_{i,h} )^\top + I_d \,.
\end{equation*}

\subsubsection{Confidence Set}

\begin{restatable}{lmm}{LemmaExplorationConfidenceSet}\label{lemma:confidence set exploration}
For any $\delta\in(0,1]$, set $\lambda_0=\sqrt{d\log(\tau H/\delta)}$, for any $t\in\llbracket T \rrbracket$ we have with probability $1-\delta$ that for all $ h\in\llbracket H \rrbracket $:
\begin{equation*}
    \lVert \theta_h^* -\hat \theta_{t,h} \rVert_{A_{t,h}} \le (1+3\sqrt{2}) \left[(B+3) \sqrt{d \log( t H/\delta) } + (1 + t)^{1/4} \left(\dfrac{\kappa}{\rho}\right)^{3/2} d\log\left( 1 + \dfrac{t}{d} \right) \right] =:\beta^0_t(\delta) \,.
\end{equation*}
\end{restatable}

\begin{proof}
\textbf{Step 1: Intermediate Confidence Set.}

Since $ \hat \theta_{\tau+1,h} $ minimises $\mathcal{\tilde L}_{t,h}(\theta) $ we have $ \cG_{\tau+1, h}(\hat \theta_{\tau+1,h}) = 0_d $.
Using the triangle inequality we get:
\begin{align*}
    &\left\lVert \cG_{\tau+1, h}(\theta_h^*) - \cG_{\tau+1, h}(\hat \theta_{\tau+1,h}) \right\rVert_{W_{\tau+1, h}^{-1}(\theta_h^*)} \\
    &\le \left\lVert \sum_{t=1}^\tau \sum_{s'\in\cS_{t,h}} (p_{t,h}^{s'}(\theta_h^*) - y_{t,h}^{s'}) \phi_{t,h}^{s'} \right\rVert_{W_{\tau+1, h}^{-1}(\theta_h^*)} + \sqrt{\lambda_0} B \\
    &\qquad+ \left\lVert \sum_{t=1}^{\tau} \lVert \phi(\tilde s'_{t,h} | \tilde s_{t,h} , a_{t,h}) \rVert_{\tilde U_{t+1,h}^{-1}} \left(p_{\tilde s_{t,h}, a_{t,h}}^{\tilde s'_{t,h}} (\theta) -1 \right) \phi(\tilde s'_{t,h} | \tilde s_{t,h} , a_{t,h}) \right\rVert_{W_{\tau+1, h}^{-1}(\theta_h^*)} \\
    &\le \left\lVert \sum_{t=1}^\tau \sum_{s'\in\cS_{t,h}} (p_{t,h}^{s'}(\theta_h^*) - y_{t,h}^{s'}) \phi_{t,h}^{s'} \right\rVert_{ (\cH_{\tau+1, h}^0(\theta_h^*))^{-1}} + \sqrt{\lambda_0} B \\
    &\qquad+ \left\lVert \sum_{t=1}^{\tau} \lVert \phi(\tilde s'_{t,h} | \tilde s_{t,h} , a_{t,h}) \rVert_{\tilde U_{t+1,h}^{-1}} \left(p_{\tilde s_{t,h}, a_{t,h}}^{\tilde s'_{t,h}} (\theta) -1 \right) \phi(\tilde s'_{t,h} | \tilde s_{t,h} , a_{t,h}) \right\rVert_{U_{\tau+1, h}^{-1}(\theta_h^*)} \,.
\end{align*}
Using a Bernstein-type concentration inequality \citep[Theorem~4]{perivier2022dynamic}, \citet[Lemma~1]{li2024provably} bounds the first two terms and yields with probability $1-\delta/H$:
\begin{equation*}
    \left\lVert \sum_{t=1}^\tau \sum_{s'\in\cS_{t,h}} (p_{t,h}^{s'}(\theta_h^*) - y_{t,h}^{s'}) \phi_{t,h}^{s'} \right\rVert_{ (\cH_{\tau+1, h}^0(\theta_h^*))^{-1} } + \sqrt{\lambda_0} B \le (B+3) \sqrt{d \log( \tau H/\delta) } \,.
\end{equation*}
We bound the last term using a Trace-Determinant argument. We start by applying the triangle inequality.
\begin{align*}
    &\left\lVert \sum_{t=1}^{\tau} \lVert \phi(\tilde s'_{t,h} | \tilde s_{t,h} , a_{t,h}) \rVert_{\tilde U_{t+1,h}^{-1}} \left(p_{\tilde s_{t,h}, a_{t,h}}^{\tilde s'_{t,h}} (\theta_h^*) -1 \right) \phi(\tilde s'_{t,h} | \tilde s_{t,h} , a_{t,h}) \right\rVert_{U_{\tau+1, h}^{-1}(\theta_h^*)} \\
    &\le \sum_{t=1}^{\tau} \left\lVert \lVert \phi(\tilde s'_{t,h} | \tilde s_{t,h} , a_{t,h}) \rVert_{\tilde U_{t+1,h}^{-1}} \left(p_{\tilde s_{t,h}, a_{t,h}}^{\tilde s'_{t,h}} (\theta_h^*) -1 \right) \phi(\tilde s'_{t,h} | \tilde s_{t,h} , a_{t,h}) \right\rVert_{U_{\tau+1, h}^{-1}(\theta_h^*)} \\
    &\le \sum_{t=1}^{\tau} \lVert \phi(\tilde s'_{t,h} | \tilde s_{t,h} , a_{t,h}) \rVert_{\tilde U_{t+1,h}^{-1}} \left| p_{\tilde s_{t,h}, a_{t,h}}^{\tilde s'_{t,h}} (\theta_h^*) -1 \right| \left\lVert \phi(\tilde s'_{t,h} | \tilde s_{t,h} , a_{t,h}) \right\rVert_{U_{\tau+1, h}^{-1}(\theta_h^*)} \\
    &\le \sum_{t=1}^{\tau} \lVert \phi(\tilde s'_{t,h} | \tilde s_{t,h} , a_{t,h}) \rVert_{\tilde U_{t+1,h}^{-1}} \left\lVert \phi(\tilde s'_{t,h} | \tilde s_{t,h} , a_{t,h}) \right\rVert_{U_{\tau+1, h}^{-1}(\theta_h^*)} && \left| p_{\tilde s_{t,h}, a_{t,h}}^{\tilde s'_{t,h}} (\theta_h^*) -1 \right| \le 1 \\
    &\le \sum_{t=1}^{\tau} \lVert \phi(\tilde s'_{t,h} | \tilde s_{t,h} , a_{t,h}) \rVert_{\tilde U_{t+1,h}^{-1}} \left\lVert \phi(\tilde s'_{t,h} | \tilde s_{t,h} , a_{t,h}) \right\rVert_{U_{t+1, h}^{-1}(\theta_h^*)} &&(U_{t+1, h} \preccurlyeq U_{\tau+1, h})
\end{align*}
Using that $ \tilde U_{t+1,h} \preccurlyeq (1+t) I_d $ we have that $ \lVert \phi(\tilde s'_{t,h} | \tilde s_{t,h} , a_{t,h}) \rVert_{\tilde U_{t+1,h}^{-1}} \ge \tfrac{1}{\sqrt{1 +  t}} \lVert \phi(\tilde s'_{t,h} | \tilde s_{t,h} , a_{t,h}) \rVert_2 \ge \tfrac{\rho}{\sqrt{1 +  t}} $. Therefore
\begin{align*}
    U_{t+1,h}(\theta_h^*) &\succcurlyeq \dfrac{\rho}{\sqrt{1 +  t}} \sum_{i=1}^{t-1} p_{\tilde s_{i,h}, a_{i,h} }^{\tilde s'_{i,h} }(\theta) (1-p_{\tilde s_{i,h}, a_{i,h} }^{\tilde s'_{i,h} }(\theta)) \phi(\tilde s'_{i,h} | \tilde s_{i,h}, a_{i,h} ) \phi(\tilde s'_{i,h} | \tilde s_{i,h}, a_{i,h} )^\top +  I_d \\
    &\succcurlyeq \dfrac{\rho }{\kappa \sqrt{1 +  t}} \tilde U_{t+1,h} \,.
\end{align*}
Hence 
\begin{align*}
    &\left\lVert \sum_{t=1}^{\tau} \lVert \phi(\tilde s'_{t,h} | \tilde s_{t,h} , a_{t,h}) \rVert_{\tilde U_{t+1,h}^{-1}} \left(p_{\tilde s_{t,h}, a_{t,h}}^{\tilde s'_{t,h}} (\theta) -1 \right) \phi(\tilde s'_{t,h} | \tilde s_{t,h} , a_{t,h}) \right\rVert_{U_{\tau+1, h}^{-1}(\theta_h^*)} \\
    &\qquad\le \sum_{t=1}^{\tau} \dfrac{(1 + t)^{1/4} \sqrt{\kappa}}{\sqrt{\rho}} \lVert \phi(\tilde s'_{t,h} | \tilde s_{t,h} , a_{t,h}) \rVert_{\tilde U_{t+1,h}^{-1}}^2 \\
    &\qquad\le \dfrac{(1 + \tau)^{1/4} \sqrt{\kappa}}{\sqrt{\rho}} \sum_{t=1}^{\tau} \lVert \phi(\tilde s'_{t,h} | \tilde s_{t,h} , a_{t,h}) \rVert_{\tilde U_{t+1,h}^{-1}}^2 \\
    &\qquad\le (1 + \tau)^{1/4} \left(\dfrac{\kappa}{\rho}\right)^{3/2} d\log\left( 1 + \dfrac{\tau}{d} \right)
\end{align*}
where the last inequality is due to \cite[Lemma~10]{abbasi2011improved}.
By applying a union bound over $h\in\llbracket H \rrbracket$ we obtain that with probability $1-\delta$ we have for all $h\in\llbracket H \rrbracket$:
\begin{equation*}
    \left\lVert \cG_{\tau+1, h}(\theta_h^*) - \cG_{\tau+1, h}(\hat \theta_{\tau+1,h}) \right\rVert_{W_{\tau+1, h}^{-1}(\theta_h^*)} \le (B+3) \sqrt{d \log( \tau H/\delta) } + (1 + \tau)^{1/4} \left(\dfrac{\kappa}{\rho}\right)^{3/2} d\log\left( 1 + \dfrac{\tau}{d} \right) \,.
\end{equation*}

\textbf{Step 2: Convex Relaxation with Self-Concordance.}

As pointed out by \citet[Section~6]{abeille2021instance}, the set $ \left\{ \theta\in\R^d : \left\lVert \cG_{\tau+1, h}(\theta) - \cG_{\tau+1, h}(\hat \theta_{\tau+1,h}) \right\rVert_{W_{\tau+1, h}^{-1}(\theta)} \le \beta^0 \right\} $ is not convex. Thus we need to build a convex relaxation of it.
Let us define 
\begin{equation*}
    G_{t,h}(\theta_1, \theta_2) := \int_0^1 \nabla^2 \mathcal{\tilde L}_{t,h}(\theta_1 + v(\theta_2 - \theta_1)) dv \,.
\end{equation*}
Hence we have
\begin{equation*}
    \cG_{t,h}(\theta_1) - \cG_{t,h}(\theta_2) = \nabla \mathcal{\tilde L}_{t,h}(\theta_1) - \nabla \mathcal{\tilde L}_{t,h}(\theta_2) = G_{t,h}(\theta_1, \theta_2) (\theta_1 - \theta_2) \,.
\end{equation*}
\cite[Proposition~1]{li2024provably} show that the objective function $ \mathcal{\tilde L}_{\tau+1,h} $ is $3\sqrt{2}$-self-concordant. By leveraging the self-concordance property twice \citep[Corollary~2]{sun2019generalized} and using $ (1-\exp(u))/u \ge (1+u)^{-1} $ we get
\begin{align*}
    \lVert \theta_h^* - \hat \theta_{\tau+1,h} \rVert_{A_{\tau+1,h}} &\le \lVert \theta_h^* - \hat \theta_{\tau+1,h} \rVert_{W_{\tau+1,h}(\theta_h^*)} &&A_{\tau+1,h}\preccurlyeq W_{\tau+1,h}(\theta_h^*) \\
    &\le (1+3\sqrt{2})^{1/2} \lVert \theta_h^* - \hat \theta_{\tau+1,h} \rVert_{G_{\tau+1,h}(\theta_h^*, \hat \theta_{\tau+1,h})} \\
    &= (1+3\sqrt{2})^{1/2} \lVert \cG_{\tau+1,h}(\theta_h^*) - \cG_{\tau+1,h}(\hat \theta_{\tau+1,h}) \rVert_{G_{\tau+1,h}^{-1}(\theta_h^*, \hat \theta_{\tau+1,h})} \\
    &\le (1+3\sqrt{2}) \lVert \cG_{\tau+1,h}(\theta_h^*) - \cG_{\tau+1,h}(\hat \theta_{\tau+1,h}) \rVert_{W_{\tau+1,h}^{-1}(\theta_h^*)} \,.
\end{align*}

\end{proof}

\subsubsection{Proof of Lemma~\ref{lemma:constant diameter}}\label{appendix:proof lemma constant diameter}

\LemmaConstantDiameter*

\begin{proof}
Fix $h\in\llbracket H \rrbracket$. We start by applying the Cauchy-Schwarz inequality:
\begin{align*}
    \diam(\Theta_h) &:= \max_{a\in\cA} \max_{s\in\cS} \max_{s'\in\cS_{h,s,a}} \max_{\theta_1, \theta_2 \in \Theta_h} | (\theta_1 - \theta_2)^\top \phi(s'|s,a) | \\
    &\le \max_{a\in\cA} \max_{s\in\cS} \max_{s'\in\cS_{h,s,a}} \max_{\theta_1, \theta_2 \in \Theta_h} \lVert \theta_1 - \theta_2 \rVert_{A_{\tau+1,h}} \lVert \phi(s'|s,a) \rVert_{A_{\tau+1,h}^{-1}} \\
    &\le 2 \beta^0_\tau(\delta) \max_{a\in\cA} \max_{s\in\cS} \max_{s'\in\cS_{h,s,a}} \lVert \phi(s'|s,a) \rVert_{A_{\tau+1,h}^{-1}} &&\text{(Lemma~\ref{lemma:confidence set exploration})} \\
    &= 2 \beta^0_\tau(\delta) \left[ \max_{a\in\cA} \max_{s\in\cS} \max_{s'\in\cS_{h,s,a}} \lVert \phi(s'|s,a) \rVert_{A_{\tau+1,h}^{-1}}^2  \right]^{1/2} \\
    &= 2 \beta^0_\tau(\delta) \tau^{-1/2} \left[ \sum_{t=1}^\tau \max_{a\in\cA} \max_{s\in\cS} \max_{s'\in\cS_{h,s,a}} \lVert \phi(s'|s,a) \rVert_{A_{\tau+1,h}^{-1}}^2  \right]^{1/2} \,.
\end{align*}
We now use that $ A_{\tau+1,h} \succcurlyeq \tilde U_{t,h} $ for all $t\in \llbracket \tau \rrbracket$ and the definition of $ (\tilde s_{t,h}, a_{t,h}, \tilde s'_{t,h}) $ to get
\begin{align*}
    \diam(\Theta_h) &\le 2 \beta^0_\tau(\delta) \tau^{-1/2} \left[ \sum_{t=1}^\tau \max_{a\in\cA} \max_{s\in\cS} \max_{s'\in\cS_{h,s,a}} \lVert \phi(s'|s,a) \rVert_{\tilde U_{t,h}^{-1}}^2  \right]^{1/2} \\
    &\le 2 \beta^0_\tau(\delta) \tau^{-1/2} \left[ \sum_{t=1}^\tau \lVert \phi(\tilde s'_{t,h}|\tilde s_{t,h} ,a_{t,h} ) \rVert_{\tilde U_{t,h}^{-1}}^2  \right]^{1/2} \\
    &= 2 \beta^0_\tau(\delta) \tau^{-1/2} \left(\dfrac{\kappa}{\rho}\right)^{1/2} \left[ \sum_{t=1}^\tau \rho\kappa \lVert \phi(\tilde s'_{t,h}|\tilde s_{t,h} ,a_{t,h} ) \rVert_{\tilde U_{t,h}^{-1}}^2  \right]^{1/2} \\
    &\le 4 \beta^0_\tau(\delta) \tau^{-1/2} \left(\dfrac{\kappa}{\rho}\right)^{1/2} \left[ d \log\left( 1 + \dfrac{\tau}{d} \right) \right]^{1/2} \,. &&\text{\citep[Lemma~9]{faury2022jointly}}
\end{align*}
Thus choosing 
\begin{equation*}
    \tau = 4^5 \left(\dfrac{\kappa}{\rho}\right)^{8} (1+3\sqrt{2})^4 d^6 (B+3)^2 \log\left(\dfrac{TH}{\delta}\right) \left[\log\left(1+\dfrac{T}{d}\right)\right]^6
\end{equation*}
concludes the proof.

\end{proof}

\subsection{Proof of Theorem~\ref{theorem:regret bound}}\label{appendix:proof regret bound}

In this section, we recall and prove our regret bound.
Before stating our result, we introduce a constant $\nu_{\max}>0$, which controls the variance term in the regret analysis without affecting the asymptotic behaviour of the bound. Formally for all stages $h\ge 1$ and states $ s\in\cS $
\begin{equation*}
    \nu_h(s) := \begin{cases}
    \dfrac{2\max_{a_1, a_2 \in \cA} \left| \Var_{a_1}\left(V_{h+1}^{*}\right) - \Var_{a_2}\left(V_{h+1}^{*}\right) \right|}{ \max_{a \in \cA} Q_h^*(s,a) - \max_{a \notin \cA_h^*(s)} Q_h^*(s,a) } & \text{if } \cA \setminus \cA_h^*(s) \neq \emptyset \\
    0 & \text{otherwise}
    \end{cases}
\end{equation*}
where $\cA_h^*(s) = \argmax_{a\in\cA} Q_h^*(s,a)$ is the set of optimal actions at state $s$.
We define $\nu_{\max} := \max_{h,s\in \llbracket H \rrbracket \times \cS} \nu_h(s)$. When $\cA$ is finite, $\nu_{\max}$ is finite.  Note that $\nu_{\max}$ plays no role in the leading term of the regret bound, only appearing in lower-order terms of order $O(T^{1/4})$ and $O(1)$ in Theorem~\ref{theorem:regret bound}. Extending the analysis to continuous action spaces, where the suboptimality gap may approach zero, is left for future work.

By construction, any action maximising the combined Q-value and variance must strictly belong to $\cA_h^*(s)$. Thus, we can choose an optimal policy $\pi^*$ defined as:
\begin{equation}\label{eq:well state def pi star}
    \pi_h^*(s) \in \argmax_{a\in\cA_h^*(s)} \left( \nu_{\max} Q_h^*(s,a) + \Var_{s'\sim p_h^{*}(\cdot|s,a)} \left[ V_{h+1}^{*}(s') \right] \right) \,.
\end{equation}
\cite[Appendix~F.4]{boudart2026enjoying} proved that this optimal policy exists. Moreover it can be computed recursively.
We now introduce our regret bound.

\TheoremRegretBound*

\begin{proof}
Throughout the proof we assume that for all $t,h\ge 1$:
\begin{equation*}
    \diam_{\cA, \cS, h}(\Theta_h) \le \dfrac{1}{3\sqrt{2}} \qquad \text{and} \qquad
    \lVert \theta_h^* - \htheta_{t,h} \rVert_{\cH_{t,h}} \le \beta_t(\delta)
\end{equation*}
which is verified with probability $1-2\delta$ thanks to Lemma~\ref{lemma:constant diameter} and Lemma~\ref{lemma:confidence set learning phase}, we apply a union bound at the end.
The regret of the exploration phase is smaller than 
\begin{equation*}
    \tau H \le 4^5 (1+3\sqrt{2})^4 \dfrac{\kappa^8}{\rho^8} d^6 H (B+3)^2 \log(TH/\delta) \log^6 T \,.
\end{equation*}
Let us now focus on the learning phase of the algorithm.

\textbf{Step 1: Regret Decomposition.}

We decompose the regret as follows
\begin{align*}
    \Reg_T :&= \sumT V_{t,1}^{\pi^*, p^{*}} (s_{t,1}) - V_{t,1}^{\pi_t, p^{*}}(s_{t,1}) \\
    &= \underbrace{\sumT V_{t,1}^{\pi^*, p^{*}} (s_{t,1}) - \tilde V_{t,1} (s_{t,1}) }_{\Reg_T(\mathrm{Policy})}
    + \underbrace{\sumT \tilde V_{t,1} (s_{t,1}) - V_{t,1}^{\pi_t, p^{*}}(s_{t,1}) }_{\Reg_T(\MDP)}
\end{align*}

\textbf{Step 2: Upper Bounding $ \Reg_T(\mathrm{Policy}) $.}

By the definition of $\tilde V_{t,h}$ (see Equation~\eqref{eq:definition optimistic reward}) and because $\theta_h^* \in \cC_{t,h}(\delta)$, for all $t,h\ge 1$ we have $ \tilde V_{t,h}(\cdot) \ge V_{h}^*(\cdot) $. Thus
\begin{equation*}
    \Reg_T (\mathrm{Policy}) := V_{1}^{*}(s_{t,1}) - \tilde V_{t,1}(s_{t,1}) \le 0 \,.
\end{equation*}

\textbf{Step 3: Upper Bounding $\Reg_T(\MDP)$.}

Let us denote by $\hat\bbQ_t $ the push operator associated to $\hat q_t$.
Let us study the instantaneous regret at time $t$, we have that
\begin{align*}
    \tilde V_{t,1} &(s_{t,1}) - V_{t,1}^{\pi_t, p^{*}} (s_{t,1}) 
    = \tilde Q_{t,1} (s_{t,1}, \pi_{t,1}(s_{t,1})) - Q_{t,1}^{\pi_t, p^{*}} (s_{t,1}, \pi_{t,1}(s_{t,1})) \\
    &= r_1(s_{t,1}, \pi_{t,1}(s_{t,1})) + \hat \bbQ_{t,1} \tilde V_{t,2} (s_{t,1}, \pi_{t,1}(s_{t,1})) - r_1(s_{t,1}, \pi_{t,1}(s_{t,1})) - \bbP_1 V_{t,2}^{\pi_t, p^{*}} (s_{t,1}, \pi_{t,1}(s_{t,1})) \\
    &= \hat\bbQ_{t,1} \tilde V_{t,2} (s_{t,1}, \pi_{t,1}(s_{t,1})) - \bbP_1 V_{t,2}^{\pi_t, p^{*}} (s_{t,1}, \pi_{t,1}(s_{t,1})) \,.
\end{align*}
We add and subtract two terms consecutively
\begin{align*}
    \tilde V_{t,1} &(s_{t,1}) - V_{t,1}^{\pi_t, p^{*}} (s_{t,1}) 
    = \bbP_1 \left( \tilde V_{t,2} - V_{t,2}^{\pi_t, p^{*}} \right) (s_{t,1}, \pi_{t,1}(s_{t,1})) + (\hat \bbQ_{t,1} - \bbP_1) \tilde V_{t,2} (s_{t,1}, \pi_{t,1}(s_{t,1})) \\
    &= \underbrace{ \bbP_1 \left( \tilde V_{t,2} - V_{t,2}^{\pi_t, p^{*}} \right) (s_{t,1}, \pi_{t,1}(s_{t,1})) - \left( \tilde V_{t,2} - V_{t,2}^{\pi_t, p^{*}} \right) (s_{t,2}) }_{\cM_{t,1}} \\
    &\qquad + \underbrace{ \left(\tilde V_{t,2} - V_{t,2}^{\pi_t, p^{*}} \right) (s_{t,2}) }_{\text{recursive term on }h} + (\hat\bbQ_{t,1} - \bbP_1) \tilde V_{t,2} (s_{t,1}, \pi_{t,1}(s_{t,1})) \,.
\end{align*}
Define
\begin{equation*}
    \cM_{t,h} := \bbP_h \left( \tilde V_{t,h+1} - V_{t,h+1}^{\pi_t, p^{*}} \right) (s_{t,h}, \pi_{t,h}(s_{t,h})) - \left( \tilde V_{t,h+1} - V_{t,h+1}^{\pi_t, p^{*}} \right) (s_{t,h+1}) \,.
\end{equation*}
Applying this recursively on $h$ we have
\begin{equation*}
    \Reg_T(\MDP) = \sumT \sumH (\hat\bbQ_{t,h} - \bbP_h) \tilde V_{t,h+1} (s_{t,h}, \pi_{t,h}(s_{t,h})) + \sumT \sumH \cM_{t,h} \,.
\end{equation*}
As $\cM_{t,h}$ is a martingale difference sequence with $\cM_{t,h} \le 2H$, we can apply Azuma-Hoeffding inequality to obtain with probability $1-\delta$
\begin{equation*}
    \Reg_T(\MDP) \le \sumT \sumH (\hat\bbQ_{t,h} - \bbP_h) \tilde V_{t,h+1} (s_{t,h}, \pi_{t,h}(s_{t,h})) + H^{3/2} \sqrt{2T\log(2/\delta)} \,.
\end{equation*}
We now add and subtract the true optimal state-value function $V^*_h$:
\begin{align}
    \Reg_T(\MDP) &\le \sumT \sumH \sum_{s'\in\cS_{t,h}} \left( \hat q_{t,h}(s'|s_{t,h}, \pi_{t,h}(s_{t,h})) - p_{h}^{*}(s'|s_{t,h}, \pi_{t,h}(s_{t,h})) \right) V^*_{t,h+1}(s') \nonumber \\
    &\quad + \sumT \sumH \sum_{s'\in\cS_{t,h}} \left( \hat q_{t,h}(s'|s_{t,h}, \pi_{t,h}(s_{t,h})) - p_{h}^{*}(s'|s_{t,h}, \pi_{t,h}(s_{t,h})) \right)  \left( \tilde V_{t,h+1}(s') - V^*_{t,h+1}(s') \right) \nonumber \\
    &\quad + H^{3/2} \sqrt{2T\log(2/\delta)} \,. \label{eq:proof regret bound 2}
\end{align}
We tackle the first sum.
Let us now define $ \Gamma_{t,h,s,a} : \R^{N_{h,s,a}} \to \R $ for $t,h,s,a \in \llbracket T \rrbracket \times \llbracket H \rrbracket \times \cS \times \cA $ as
\begin{equation*}
    \Gamma_{t,h,s,a}(u) := \sum_{s'\in\cS_{h,s,a}} \dfrac{\exp(u_{s'}) }{1 + \sum_{s''\in\cS_{h,s,a}} \exp(u_{s''}) } V_{t,h+1}^{*} (s') \,.
\end{equation*}
We denote $\Gamma_{t,h} := \Gamma_{t,h,s_{t,h}, \pi_{t,h}(s_{t,h})}$.
Moreover, we denote $ u_{h,s,a}^{p^{*}}, u_{h,s,a}^{\hat q_t} \in \R^{N_{h,s,a}} $ the vectors respectively defined by $ [ u_{h,s,a}^{p^{*}} ]_{s'} = (\theta_h^*)^\top \phi_{s,a}^{s'} $ and $ [ u_{h,s,a}^{\hat q_t} ]_{s'} $ the logits associated to $ \hat q_{t,h}(s'|s,a) $. We rearrange the terms in the first sum of Equation~\eqref{eq:proof regret bound 2} as follows:
\begin{multline*}
    \sumT \sumH \sum_{s'\in\cS_{t,h}} \left( q_{t,h}^*(s'|s_{t,h}, \pi_{t,h}(s_{t,h})) - p_{h}^{*}(s'|s_{t,h}, \pi_{t,h}(s_{t,h})) \right) V^*_{t,h+1}(s') \\
    = \sumT \sumH \Gamma_{t,h} \left( u_{h,s_{t,h},\pi_{t,h}(s_{t,h})}^{\hat q_t} \right) - \Gamma_{t,h} \left( u_{h,s_{t,h},\pi_{t,h}(s_{t,h})}^{p^{*}} \right) \,.
\end{multline*}

We apply a Taylor decomposition on $ \Gamma_{t,h} \left( u_{h,s_{t,h},\pi_{t,h}(s_{t,h})}^{\hat q_t} \right) $, there exists $ \bar u_{h,s_{t,h},\pi_{t,h}(s_{t,h})} := \nu u_{h,s_{t,h},\pi_{t,h}(s_{t,h})}^{\hat q_t} + (1-\nu) u_{h,s_{t,h},\pi_{t,h}(s_{t,h})}^{p^{*}} $ with $\nu \in [0,1] $ such that:
\small
\begin{align}
    \Gamma&_{t,h} \left( u_{h,s_{t,h},\pi_{t,h}(s_{t,h})}^{\hat q_t} \right) 
    = \Gamma_{t,h} \left( u_{h,s_{t,h},\pi_{t,h}(s_{t,h})}^{p^{*}} \right) \nonumber \\
    &+ \nabla \Gamma_{t,h} \left( u_{h,s_{t,h},\pi_{t,h}(s_{t,h})}^{p^{*}} \right)^\top \left( u_{h,s_{t,h},\pi_{t,h}(s_{t,h})}^{\hat q_t} - u_{h,s_{t,h},\pi_{t,h}(s_{t,h})}^{p^{*}} \right) \nonumber \\
    &+ \dfrac{1}{2} \left( u_{h,s_{t,h},\pi_{t,h}(s_{t,h})}^{\hat q_t} - u_{h,s_{t,h},\pi_{t,h}(s_{t,h})}^{p^{*}} \right)^\top \nabla^2 \Gamma_{t,h} \left( \bar u_{h,s_{t,h},\pi_{t,h}(s_{t,h})} \right) \left( u_{h,s_{t,h},\pi_{t,h}(s_{t,h})}^{\hat q_t} - u_{h,s_{t,h},\pi_{t,h}(s_{t,h})}^{p^{*}} \right) \,. \label{eq:proof regret bound well 3}
\end{align}
\normalsize
Let us first write the gradient of $\Gamma$. Let $s'\in \cS_{t,h} $
\begin{align*}
    \left[ \nabla \Gamma_{t,h} \left( u_{h,s_{t,h},\pi_{t,h}(s_{t,h})}^{p^{*}} \right) \right]_{s'} &= 
    p_{s_{t,h}, \pi_{t,h}(s_{t,h})}^{s'} (\theta_h^*) V_{t,h+1}^{*}(s') \\
    &\quad- \sum_{s''\in\cS_{t,h}} p_{s_{t,h}, \pi_{t,h}(s_{t,h})}^{s'} (\theta_h^*) p_{s_{t,h}, \pi_{t,h}(s_{t,h})}^{s''} (\theta_h^*) V_{t,h+1}^{*}(s'') \\
    &= p_{s_{t,h}, \pi_{t,h}(s_{t,h})}^{s'} (\theta_h^*) \left( V_{t,h+1}^{*}(s') - \E_{s''\sim p^{*}} \left[ V_{t,h+1}^{*}(s'') | s_{t,h},\pi_{t,h}(s_{t,h}) \right] \right) \,.
\end{align*}
A crucial property of the gradient of the softmax is that its components sum to 0. Indeed
\begin{align*}
    \sum_{s'\in\cS_{t,h}} &\left[ \nabla \Gamma_{t,h} \left( u_{h,s_{t,h},\pi_{t,h}(s_{t,h})}^{p^{*}} \right) \right]_{s'} \\
    &= \sum_{s'\in\cS_{t,h}} p_{s_{t,h}, \pi_{t,h}(s_{t,h})}^{s'} (\theta_h^*) \left( V_{t,h+1}^{*}(s') - \E_{s''\sim p^{*}} \left[ V_{t,h+1}^{*}(s'') | s_{t,h},\pi_{t,h}(s_{t,h}) \right] \right) \\
    &= \sum_{s'\in\cS_{t,h}} p_{s_{t,h}, \pi_{t,h}(s_{t,h})}^{s'} (\theta_h^*) V_{t,h+1}^{*}(s') - \E_{s''\sim p^{*}} \left[ V_{t,h+1}^{*}(s'') | s_{t,h},\pi_{t,h}(s_{t,h}) \right] \\
    &= 0 \,.
\end{align*}
Because the components sum to 0, we can subtract $ \bar \phi_{t,h} = \sum_{s'\in\cS_{t,h}} p_{s_{t,h}, \pi_{t,h}(s_{t,h})}^{s'} (\theta_h^*) \phi_{s_{t,h},a_{t,h}}^{s'} $ inside the inner product without altering the result:
\begin{align*}
    &\nabla \Gamma_{t,h}\left( u_{h,s_{t,h},\pi_{t,h}(s_{t,h})}^{p^{*}} \right)^\top \left( u_{h,s_{t,h},\pi_{t,h}(s_{t,h})}^{\hat q_t} - u_{h,s_{t,h},\pi_{t,h}(s_{t,h})}^{p^{*}} \right) \\
    &\qquad= \sum_{s'\in\cS_{t,h}}  \left[ \nabla \Gamma_{t,h} \left( u_{h,s_{t,h},\pi_{t,h}(s_{t,h})}^{p^{*}} \right) \right]_{s'} \left( \tilde \theta_{t,h} - \theta_h^* \right)^\top \phi_{s_{t,h}, a_{t,h}}^{s'} \\
    &\qquad= \sum_{s'\in\cS_{t,h}}  \left[ \nabla \Gamma_{t,h} \left( u_{h,s_{t,h},\pi_{t,h}(s_{t,h})}^{p^{*}} \right) \right]_{s'} \left( \tilde \theta_{t,h} - \theta_h^* \right)^\top \left( \phi_{s_{t,h}, a_{t,h}}^{s'} -  \bar \phi_{t,h} \right) \\
    &\qquad\le \sum_{s'\in\cS_{t,h}} p_{s_{t,h}, a_{t,h}}^{s'} (\theta_h^*) \left| V_{t,h+1}^{*}(s') - \E_{s''\sim p^{*}} \left[ V_{t,h+1}^{*}(s'') \right] \right| \; \lVert \tilde \theta_{t,h} - \theta_h^* \rVert_{\cH_{t,h}} \lVert \phi_{s_{t,h}, a_{t,h}}^{s'} - \bar \phi_{t,h} \rVert_{\cH_{t,h}^{-1}} \\
    &\qquad \le 2 \beta_t(\delta) \sum_{s'\in\cS_{t,h}} p_{s_{t,h}, a_{t,h}}^{s'} (\theta_h^*) \left| V_{t,h+1}^{*}(s') - \E_{s''\sim p^{*}} \left[ V_{t,h+1}^{*}(s'') \right] \right| \; \lVert \phi_{s_{t,h}, a_{t,h}}^{s'} - \bar \phi_{t,h} \rVert_{\cH_{t,h}^{-1}} \\
    &\qquad \le 2 \beta_t(\delta) \left[ \sum_{s'\in\cS_{t,h}} p_{s_{t,h}, a_{t,h}}^{s'} (\theta_h^*) \left| V_{t,h+1}^{*}(s') - \E_{s''\sim p^{*}} \left[ V_{t,h+1}^{*}(s'') \right] \right|^2 \right]^{1/2} \\
    &\qquad\qquad \cdot \left[ \sum_{s'\in\cS_{t,h}} p_{s_{t,h}, a_{t,h}}^{s'} (\theta_h^*) \lVert \phi_{s_{t,h}, a_{t,h}}^{s'} - \bar \phi_{t,h} \rVert_{\cH_{t,h}^{-1}}^2 \right]^{1/2}
\end{align*}
where the first and last inequalities are due to Cauchy-Schwarz inequality, and the second to Lemma~\ref{lemma:confidence set learning phase}.
By substituting into Equation~\eqref{eq:proof regret bound well 3} we obtain
\small
\begin{align*}
    \Gamma&_{t,h} \left( u_{h,s_{t,h},\pi_{t,h}(s_{t,h})}^{\hat q_t} \right) 
    = \Gamma_{t,h} \left( u_{h,s_{t,h},\pi_{t,h}(s_{t,h})}^{p^{*}} \right) \\
    &+2\beta_t(\delta) \left[ \sum_{s'\in\cS_{t,h}} p_{s_{t,h}, a_{t,h}}^{s'} (\theta_h^*) \left| V_{t,h+1}^{*}(s') - \E_{s''\sim p^{*}} \left[ V_{t,h+1}^{*}(s'') \right] \right|^2 \right]^{1/2} \left[ \sum_{s'\in\cS_{t,h}} p_{s_{t,h}, a_{t,h}}^{s'} (\theta_h^*) \lVert \phi_{s_{t,h}, a_{t,h}}^{s'} - \bar \phi_{t,h} \rVert_{\cH_{t,h}^{-1}}^2 \right]^{1/2} \\
    &+ \dfrac{1}{2} \left( u_{h,s_{t,h},\pi_{t,h}(s_{t,h})}^{\hat q_t} - u_{h,s_{t,h},\pi_{t,h}(s_{t,h})}^{p^{*}} \right)^\top \nabla^2 \Gamma_{t,h} \left( \bar u_{h,s_{t,h},\pi_{t,h}(s_{t,h})} \right) \left( u_{h,s_{t,h},\pi_{t,h}(s_{t,h})}^{\hat q_t} - u_{h,s_{t,h},\pi_{t,h}(s_{t,h})}^{p^{*}} \right) \,.
\end{align*}
\normalsize
We now substitute this Taylor decomposition into Equation~\eqref{eq:proof regret bound 2} and use that for all $t\ge 1$ we have $ \beta_t(\delta) \le \beta_T(\delta)$ to get
\small
\begin{align*}
    \Reg_T(\MDP) \le & 2\beta_T(\delta) \sumT \sumH  \left[ \sum_{s'\in\cS_{t,h}} p_{s_{t,h}, a_{t,h}}^{s'} (\theta_h^*) \left| V_{t,h+1}^{*}(s') - \E_{s''\sim p^{*}} \left[ V_{t,h+1}^{*}(s'') \right] \right|^2 \right]^{1/2} \\
    &\qquad \cdot \left[ \sum_{s'\in\cS_{t,h}} p_{s_{t,h}, a_{t,h}}^{s'} (\theta_h^*) \lVert \phi_{s_{t,h}, a_{t,h}}^{s'}- \bar \phi_{t,h} \rVert_{\cH_{t,h}^{-1}}^2 \right]^{1/2} \\ 
    +& \sumT \sumH \dfrac{1}{2} \left( u_{h,s_{t,h},\pi_{t,h}(s_{t,h})}^{\hat q_t} - u_{h,s_{t,h},\pi_{t,h}(s_{t,h})}^{p^{*}} \right)^\top \nabla^2 \Gamma_{t,h} \left( \bar u_{h,s_{t,h},\pi_{t,h}(s_{t,h})} \right) \\
    &\qquad \cdot \left( u_{h,s_{t,h},\pi_{t,h}(s_{t,h})}^{\hat q_t} - u_{h,s_{t,h},\pi_{t,h}(s_{t,h})}^{p^{*}} \right) \\
    +& \sumT \sumH \sum_{s'\in\cS_{t,h}} \left( \hat q_{t,h}(s'|s_{t,h}, \pi_{t,h}(s_{t,h})) - p_{h}^{*}(s'|s_{t,h}, \pi_{t,h}(s_{t,h})) \right)  \left( \tilde V_{t,h+1}(s') - V^*_{t,h+1}(s') \right) \\
    +& H^{3/2} \sqrt{2T\log(2/\delta)} \,.
\end{align*}
\normalsize
We apply Cauchy-Schwarz inequality on the first term to have
\small
\begin{align*}
    \Reg_T(\MDP) \le & 2\beta_T(\delta) \left[ \sumT \sumH \sum_{s'\in\cS_{t,h}} p_{s_{t,h}, a_{t,h}}^{s'} (\theta_h^*) \left| V_{t,h+1}^{*}(s') - \E_{s''\sim p^{*}} \left[ V_{t,h+1}^{*}(s'') \right] \right|^2 \right]^{1/2} \\
    &\qquad \cdot \left[ \sumT \sumH \sum_{s'\in\cS_{t,h}} p_{s_{t,h}, a_{t,h}}^{s'} (\theta_h^*) \lVert \phi_{s_{t,h}, a_{t,h}}^{s'}- \bar \phi_{t,h} \rVert_{\cH_{t,h}^{-1}}^2 \right]^{1/2} \\ 
    +& \sumT \sumH \dfrac{1}{2} \left( u_{h,s_{t,h},\pi_{t,h}(s_{t,h})}^{\hat q_t} - u_{h,s_{t,h},\pi_{t,h}(s_{t,h})}^{p^{*}} \right)^\top \nabla^2 \Gamma_{t,h} \left( \bar u_{h,s_{t,h},\pi_{t,h}(s_{t,h})} \right) \\
    &\qquad \cdot \left( u_{h,s_{t,h},\pi_{t,h}(s_{t,h})}^{\hat q_t} - u_{h,s_{t,h},\pi_{t,h}(s_{t,h})}^{p^{*}} \right) \\
    +& \sumT \sumH \sum_{s'\in\cS_{t,h}} \left( \hat q_{t,h}(s'|s_{t,h}, \pi_{t,h}(s_{t,h})) - p_{h}^{*}(s'|s_{t,h}, \pi_{t,h}(s_{t,h})) \right)  \left( \tilde V_{t,h+1}(s') - V^*_{t,h+1}(s') \right) \\
    +& H^{3/2} \sqrt{2T\log(2/\delta)} \,.
\end{align*}
\normalsize

\textbf{Step 4: Recovering $\Reg_T(\MDP)$ in the upper bound.}

We show that the sum of $ \tilde V_{t,h+1}(s') - V^*_{t,h+1}(s') $ is linked to the regret $\Reg_T$ and a martingale difference term. We first use the definition of $\pi^*$:
\begin{equation*}
     \tilde V_{t,h+1}(s') - V^*_{t,h+1}(s') \le  \tilde V_{t,h+1}(s') - V^{\pi_t, p^*}_{t,h+1}(s') \,.
\end{equation*}
We expand the softmax difference via the Mean Value Theorem and the gradient identity $\nabla_u \mu(u)_{s'} = \mu(u)_{s'} (e_{s'} - \mu(s'))$ where $\mu : \R^{N_{h,s,a}}\to \R$ denotes the softmax (see for instance \cite[Equation~(4.106)]{bishop2006pattern}). 
Let $(u^*_{t,h}) := (\theta_h^*)^\top (\phi_{s_{t,h},\pi_{t,h}(s_{t,h})}^{s'} - \E_{s''\sim \hat q_{t,h}} [ \phi_{s_{t,h},\pi_{t,h}(s_{t,h})}^{s''} ] ) $ and $(\hat u_{t,h}) := (\tilde \theta_{t,h})^\top (\phi_{s_{t,h},\pi_{t,h}(s_{t,h})}^{s'} - \E_{s''\sim \hat q_{t,h}} [ \phi_{s_{t,h},\pi_{t,h}(s_{t,h})}^{s''} ] ) $ be the centered logits associated respectively with $p^*$ and $\hat q_t$. Centering is harmless as the softmax is invariant under uniform logit shifts.
There exists $w$ on the segment between the logits $\hat u_{t,h}$ and $u^*_{t,h}$ such that 
\begin{equation*}
    \hat q_{t,h}(s') - p_{h}^{*}(s') = \mu(w)_{s'} \left( \left( (\hat u_{t,h})_{s'} - (u_{t,h}^*)_{s'} \right) - \E_{s''\sim\mu(w)}\left[ (\hat u_{t,h})_{s''} - (u_{t,h}^*)_{s''} \right] \right) \,.
\end{equation*}
Substituting and applying Cauchy-Schwarz inequality we get
\begin{align*}
    \sumT \sumH \sum_{s'\in\cS_{t,h}} &\left( \hat q_{t,h}(s'|s_{t,h}, \pi_{t,h}(s_{t,h})) - p_{h}^{*}(s'|s_{t,h}, \pi_{t,h}(s_{t,h})) \right)  \left( \tilde V_{t,h+1}(s') - V^*_{t,h+1}(s') \right) \\
    &\le \left[ \sumT \sumH \Var_{s'\sim\mu(w)} \left[ (\hat u_{t,h})_{s'} - (u^*_{t,h})_{s'} \right]  \right]^{1/2} \\
    &\qquad \cdot \left[ \sumT \sumH \sum_{s'\in\cS_{t,h}} \mu(w)_{s'} \left( \tilde V_{t,h+1}(s') - V^{\pi_t, p^*}_{t,h+1}(s') \right)^2 \right]^{1/2} \,.
\end{align*}
By definition of the centered logits, we have
\begin{equation*}
    \left(\hat u_{t,h} - u^*_{t,h} \right)_{s'} 
    = \left( \tilde \theta_{t,h} - \theta_h^* \right)^\top \left( \phi_{s_{t,h}, \pi_{t,h}(s_{t,h})}^{s'} - \E_{s''\sim \hat q_{t,h}} \left[ \phi_{s_{t,h},\pi_{t,h}(s_{t,h})}^{s''} \right] \right) \,.
\end{equation*}
We apply Cauchy-Schwarz inequality and get 
\begin{multline*}
     \left(\hat u_{t,h} - u^*_{t,h} \right)_{s'} 
     \le \left\lVert \tilde \theta_{t,h} - \theta_h^* \right\rVert_{\cH_{t,h}} \llVert \phi_{s_{t,h}, \pi_{t,h}(s_{t,h})}^{s'} - \E_{s''\sim \hat q_{t,h}} \left[ \phi_{s_{t,h},\pi_{t,h}(s_{t,h})}^{s''} \right] \rrVert_{\cH_{t,h}^{-1}} \\
     \le 2 \beta_t(\delta) \llVert \phi_{s_{t,h}, \pi_{t,h}(s_{t,h})}^{s'} - \E_{s''\sim \hat q_{t,h}} \left[ \phi_{s_{t,h},\pi_{t,h}(s_{t,h})}^{s''} \right] \rrVert_{\cH_{t,h}^{-1}}
\end{multline*}
where the second inequality is because $\tilde \theta_{t,h}$ and $ \theta_h^* $ lie in the confidence set $\cC_{t,h}(\delta)$. 
For any $u, w$ with $\|w - u\|_\infty \leq c$, one has $\mu(w)_{s'}/\mu(u)_{s'} \leq e^{2c}$ for all $s'$, by comparing numerator and denominator of the softmax. The exploration phase ensures $\lVert w - \hat u_{t,h} \rVert_\infty \le 1$, which thus yields $\mu(w)_{s'} \le e^2 \mu(\hat u_{t,h})_{s'}$.
Therefore, we get
\begin{multline*}
    \Var_{s'\sim\mu(w)} \left[ (\hat u_{t,h})_{s'} - (u^*_{t,h})_{s'} \right] \\
    \le 4 \beta_t(\delta)^2 \E_{s'\sim \hat q_{t,h}} \left[ \llVert \phi_{s_{t,h}, \pi_{t,h}(s_{t,h})}^{s'} - \E_{s''\sim \hat q_{t,h}} \left[ \phi_{s_{t,h},\pi_{t,h}(s_{t,h})}^{s''} \right] \rrVert_{\cH_{t,h}^{-1}}^2 | s_{t,h}, \pi_{t,h}(s_{t,h}) \right] \,.
\end{multline*}
We apply \cite[Lemma~6~(IV)]{li2024provably} and bound the sum
\begin{equation*}
    \sumT \sumH \E_{s'\sim \hat q_{t,h}} \left[ \llVert \phi_{s_{t,h}, \pi_{t,h}(s_{t,h})}^{s'} - \E_{s''\sim \hat q_{t,h}} \left[ \phi_{s_{t,h},\pi_{t,h}(s_{t,h})}^{s''} \right] \rrVert_{\cH_{t,h}^{-1}}^2 | s_{t,h}, \pi_{t,h}(s_{t,h}) \right]
    \le 2d H \log(1+T) \,.
\end{equation*}

We now tackle the sum of state-value functions.
Because $\tilde V_{t,h+1}(s') \ge V^{\pi_t, p^*}_{t,h+1}(s')$ we have $\tilde V_{t,h+1}(s') - V^{\pi_t, p^*}_{t,h+1}(s') \le H$. Thus
\begin{align}
    \sumT &\sumH \sum_{s'\in\cS_{t,h}} \left( \hat q_{t,h}(s'|s_{t,h}, \pi_{t,h}(s_{t,h})) - p_{h}^{*}(s'|s_{t,h}, \pi_{t,h}(s_{t,h})) \right)  \left( \tilde V_{t,h+1}(s') - V^*_{t,h+1}(s') \right) \nonumber \\
    &\le 2 \sqrt{2} e^2 \beta_T(\delta) \sqrt{d} H \sqrt{\log(1+T)} \nonumber \\
    &\qquad \cdot \left[ \sumT \sumH \sum_{s'\in\cS_{t,h}} p_{h}^{*}(s'|s_{t,h}, \pi_{t,h}(s_{t,h})) \left( \tilde V_{t,h+1}(s') - V^{\pi_t, p^*}_{t,h+1}(s') \right) \right]^{1/2} \,. \label{eq:proof regret bound 4}
\end{align}
By adding and subtracting a term we get
\begin{align}
    & \left(\tilde V_{t,h} - V^{\pi_t, p^*}_{t,h}\right) (s_{t,h}) \nonumber \\
    &= r_h(s_{t,h}, \pi_{t,h}(s_{t,h})) + \hat \bbQ_{t,h} \tilde V_{t,h+1}(s_{t,h}, \pi_{t,h}(s_{t,h})) - r_h(s_{t,h}, \pi_{t,h}(s_{t,h})) - \bbP_h V^{\pi_t, p^*}_{t,h+1} (s_{t,h}, \pi_{t,h}(s_{t,h})) \nonumber \\
    &= \underbrace{\left( \hat \bbQ_{t,h} - \bbP_h \right) \tilde V_{t,h+1}(s_{t,h}, \pi_{t,h}(s_{t,h}))}_{I_{t,h}} + \bbP_h \left( \tilde V_{t,h+1} - V^{\pi_t, p^*}_{t,h+1} \right) (s_{t,h}, \pi_{t,h}(s_{t,h})) \label{eq:proof regret bound 3}
\end{align}
Because $\hat q_{t,h}$ maximises $I_{t,h}$ over a set containing $p^*_h$ we have
\begin{equation*}
    \bbP_h \left( \tilde V_{t,h+1} - V^{\pi_t, p^*}_{t,h+1} \right) (s_{t,h}, \pi_{t,h}(s_{t,h})) 
    \le \left(\tilde V_{t,h} - V^{\pi_t, p^*}_{t,h} \right) (s_{t,h}) \,.
\end{equation*}
Let us define
\begin{equation*}
    \xi_{t,h+1} :=  \bbP_h \left( \tilde V_{t,h+1} - V^{\pi_t, p^*}_{t,h+1} \right) (s_{t,h}, \pi_{t,h}(s_{t,h})) 
    - \left(\tilde V_{t,h+1} - V^{\pi_t, p^*}_{t,h+1}\right) (s_{t,h+1})
\end{equation*}
the martingale difference. By adding and subtracting a term in Equation~\eqref{eq:proof regret bound 3} we have
\begin{equation*}
     \left(\tilde V_{t,h} - V^{\pi_t, p^*}_{t,h}\right) (s_{t,h}) 
     = I_{t,h} + \underbrace{ \left(\tilde V_{t,h+1} - V^{\pi_t, p^*}_{t,h+1}\right) (s_{t,h+1}) }_{\text{recursive term}} + \xi_{t,h+1} \,.
\end{equation*}
We expand the recursion from $h$ to $H$ to get
\begin{equation*}
    \left(\tilde V_{t,h} - V^{\pi_t, p^*}_{t,h} \right) (s_{t,h})
    = \sum_{k=h}^H I_{t,k} + \sum_{k=h}^H \xi_{t,k+1} \,.
\end{equation*}
Summing this over all $h\ge 1$ yields:
\begin{equation*}
    \sumH  \left(\tilde V_{t,h} - V^{\pi_t, p^*}_{t,h} \right) (s_{t,h})
    = \sumH h I_{t,h} + \sumH h \xi_{t,h+1} \le H \sumH I_{t,h} + \sumH h \xi_{t,h+1} \,.
\end{equation*}
Because $ \sumH I_{t,h} = \tilde V_{t,1}(s_{t,1}) - V^{\pi_t, p^*}_{t,1}(s_{t,1}) - \sumH \xi_{t,h+1} $, substituting this back reveals the initial gap:
\begin{equation*}
    \sumH  \left(\tilde V_{t,h} - V^{\pi_t, p^*}_{t,h} \right) (s_{t,h}) 
    = H \left( \tilde V_{t,1}(s_{t,1}) - V^{\pi_t, p^*}_{t,1}(s_{t,1}) \right) + \sumH (h-H) \xi_{t,h+1} \,.
\end{equation*}
We now bound the sum of expectations
\begin{multline*}
    \sumT \sumH \bbP_h (\tilde V_{t,h} - V^{\pi_t, p^*}_{t,h}) (s_{t,h}, \pi_{t,h}(s_{t,h})) 
    = \sumT \sumH (\tilde V_{t,h} - V^{\pi_t, p^*}_{t,h}) (s_{t,h+1}) + \xi_{t,h+1} \\
    = H \underbrace{ \sumT  \tilde V_{t,1}(s_{t,1}) - V^{\pi_t, p^*}_{t,1}(s_{t,1}) }_{\Reg_T(\MDP)} + \sumT \sumH (h-H+1) \xi_{t,h+1} \,.
\end{multline*}
The second term is a martingale with increments bounded by $H^2$. By the Azuma-Hoeffding inequality, it is bounded by $\sqrt{2} H^{5/2} \sqrt{T \log(2/\delta)}$ with probability $1-\delta$.
Substituting this back into Equation~\eqref{eq:proof regret bound 4}:
\begin{align*}
    \sumT &\sumH \sum_{s'\in\cS_{t,h}} \left( \hat q_{t,h}(s'|s_{t,h}, \pi_{t,h}(s_{t,h})) - p_{h}^{*}(s'|s_{t,h}, \pi_{t,h}(s_{t,h})) \right)  \left( \tilde V_{t,h+1}(s') - V^*_{t,h+1}(s') \right) \\
    &\le 2 \sqrt{2} e^2 \beta_T(\delta) \sqrt{d} H^{3/2} \sqrt{\log(1+T)} \\
    &\qquad \cdot \left[ \Reg_T(\MDP) + \sqrt{2} H^{3/2} \sqrt{T \log(2/\delta)} \right]^{1/2} \,.
\end{align*}
Thus $\Reg_T(\MDP)$ is upper bounded by
\small
\begin{align}
    \Reg_T(\MDP) \le & 2\beta_T(\delta) \left[ \sumT \sumH \underbrace{ \sum_{s'\in\cS_{t,h}} p_{s_{t,h}, a_{t,h}}^{s'} (\theta_h^*) \left( V_{t,h+1}^{*}(s') - \E_{s''\sim p^{*}} \left[ V_{t,h+1}^{*}(s'') \right] \right)^2 }_{\Var_{\pi_{t,h}} \left[ V_{h+1}^{\pi^*, p^{*}}(s') \right]} \right]^{1/2} \nonumber \\
    &\qquad \cdot \left[ \sumT \sumH \sum_{s'\in\cS_{t,h}} p_{s_{t,h}, a_{t,h}}^{s'} (\theta_h^*) \lVert \phi_{s_{t,h}, a_{t,h}}^{s'}- \bar \phi_{t,h} \rVert_{\cH_{t,h}^{-1}}^2 \right]^{1/2} \nonumber \\ 
    +& \sumT \sumH \dfrac{1}{2} \left( u_{h,s_{t,h},\pi_{t,h}(s_{t,h})}^{\hat q_t} - u_{h,s_{t,h},\pi_{t,h}(s_{t,h})}^{p^{*}} \right)^\top \nabla^2 \Gamma_{t,h} \left( \bar u_{h,s_{t,h},\pi_{t,h}(s_{t,h})} \right) \nonumber \\
    &\qquad \cdot \left( u_{h,s_{t,h},\pi_{t,h}(s_{t,h})}^{\hat q_t} - u_{h,s_{t,h},\pi_{t,h}(s_{t,h})}^{p^{*}} \right) \nonumber \\
    +& 2 \sqrt{2} e^2 \beta_T(\delta) \sqrt{d} H^{3/2} \sqrt{\log(1+T)} \left[ \Reg_T(\MDP) + \sqrt{2} H^{3/2} \sqrt{T \log(2/\delta)} \right]^{1/2} \nonumber \\
    +& H^{3/2} \sqrt{2T\log(2/\delta)} \,. \label{eq:proof regret bound 5}
\end{align}
\normalsize

\textbf{Step 5: Recovering $\bar \sigma_{T}$ in the upper bound.}

Using the definition of $\pi^*$ (see Equation~\eqref{eq:well state def pi star}) we upper bound the first term of Equation~\eqref{eq:proof regret bound 5} by
\begin{equation*}
    \Var_{\pi_{t,h}} \left[ V_{h+1}^{*}(s') \right] 
    \le \Var_{\pi^*} \left[ V_{h+1}^{*}(s') \right]
    + \nu_{\max} Q_h^*(s_{t,h}, \pi_h^*(s_{t,h})) - Q_h^*(s_{t,h}, \pi_{t,h}(s_{t,h})) \,.
\end{equation*}

We show that the sum of $ Q_h^*(s_{t,h}, \pi_h^*(s_{t,h})) - Q_h^*(s_{t,h}, \pi_{t,h}(s_{t,h})) $ is linked to the regret $\Reg_T(\MDP)$ and a martingale difference term. We add and subtract a term in the instantaneous regret:
\begin{align*}
    V_{1}^{*}(s_{t,1}) - V_{t,1}^{\pi_t, p^{*}}(s_{t,1})
    &= V_{1}^{*}(s_{t,1}) - Q_{1}^{*}(s_{t,1}, \pi_{t,1}(s_{t,1})) + Q_{1}^{*}(s_{t,1}, \pi_{t,1}(s_{t,1})) - Q_{t,1}^{\pi_t, p^{*}}(s_{t,1}, \pi_{t,1}(s_{t,1})) \\
    &= Q_{1}^{*}(s_{t,1}, \pi^*(s_{t,1})) - Q_{1}^{*}(s_{t,1}, \pi_{t,1}(s_{t,1})) \\
    &\qquad + r_1(s_{t,1}, \pi_{t,1}(s_{t,1})) + \bbP_1 V_{2}^{*}(s_{t,1}, \pi_{t,1}(s_{t,1})) \\
    &\qquad - r_1(s_{t,1}, \pi_{t,1}(s_{t,1})) - \bbP_1 V_{t,2}^{\pi_t, p^{*}}(s_{t,1}, \pi_{t,1}(s_{t,1})) \,.
\end{align*}
We add and subtract another term:
\begin{align*}
    V_{1}^{*}(s_{t,1}) - V_{t,1}^{\pi_t, p^{*}}(s_{t,1})
    &= Q_{1}^{*}(s_{t,1}, \pi^*(s_{t,1})) - Q_{1}^{*}(s_{t,1}, \pi_{t,1}(s_{t,1})) \\
    &\qquad + \underbrace{ \bbP_1 \left( V_{2}^{*} - V_{t,2}^{\pi_t, p^{*}} \right) (s_{t,1}, \pi_{t,1}(s_{t,1})) - \left( V_{2}^{*} - V_{t,2}^{\pi_t, p^{*}} \right) (s_{t,2}) }_{-\tilde\cM_{t,1}} \\
    &\qquad  + \left( V_{2}^{*} - V_{t,2}^{\pi_t, p^{*}} \right) (s_{t,2}) \,.
\end{align*}
We define
\begin{equation*}
    \tilde \cM_{t,h} := \left( V_{h+1}^{*} - V_{t,h+1}^{\pi_t, p^{*}} \right) (s_{t,h+1}) - \bbP_h \left( V_{h+1}^{*} - V_{t,h+1}^{\pi_t, p^{*}} \right) (s_{t,h}, \pi_{t,h}(s_{t,h})) \,.
\end{equation*}
Applying this recursively on $h$, we have
\begin{equation*}
    \sumT \sumH Q_{h}^{*}(s_{t,h}, \pi^*(s_{t,h})) - Q_{h}^{*}(s_{t,h}, \pi_{t,h}(s_{t,h})) = \Reg_T + \sumT \sumH \tilde \cM_{t,h} \,.
\end{equation*}
As $\tilde \cM_{t,h}$ is a martingale difference sequence with $ \tilde \cM_{t,h} \le 2H $ we can apply Azuma-Hoeffding inequality to obtain with probability $1-\delta$:
\begin{equation*}
    \sumT \sumH Q_{h}^{*}(s_{t,h}, \pi^*_h(s_{t,h})) - Q_{h}^{*}(s_{t,h}, \pi_{t,h}(s_{t,h})) = \Reg_T + H^{3/2} \sqrt{2T\log(2/\delta)} \,.
\end{equation*}
Using our regret decomposition from Step~1 and optimistic argument from Step~2, we upper bound the regret $\Reg_T$ by $\Reg_T(\MDP)$:
\begin{equation*}
    \sumT \sumH Q_{h}^{*}(s_{t,h}, \pi^*_h(s_{t,h})) - Q_{h}^{*}(s_{t,h}, \pi_{t,h}(s_{t,h})) = \Reg_T(\MDP) + H^{3/2} \sqrt{2T\log(2/\delta)} \,.
\end{equation*}
By substituting into Equation~\eqref{eq:proof regret bound 5} we get
\small
\begin{align*}
    \Reg_T(\MDP) \le & 2\beta_T(\delta) \Biggg[ \underbrace{ \sumT \sumH \Var_{\pi^*} \left[ V_{h+1}^{*}(s') \right]}_{H^3 T \, \bar\sigma_{T}^2} + \nu_{\max} \Reg_T(\MDP) + \nu_{\max} H^{3/2} \sqrt{2T\log(2/\delta)} \Biggg]^{1/2} \\
    &\qquad \cdot \left[ \sumT \sumH \sum_{s'\in\cS_{t,h}} p_{s_{t,h}, a_{t,h}}^{s'} (\theta_h^*) \lVert \phi_{s_{t,h}, a_{t,h}}^{s'}- \bar \phi_{t,h} \rVert_{\cH_{t,h}^{-1}}^2 \right]^{1/2}  \\ 
    +& \sumT \sumH \dfrac{1}{2} \left( u_{h,s_{t,h},\pi_{t,h}(s_{t,h})}^{\hat q_t} - u_{h,s_{t,h},\pi_{t,h}(s_{t,h})}^{p^{*}} \right)^\top \nabla^2 \Gamma_{t,h} \left( \bar u_{h,s_{t,h},\pi_{t,h}(s_{t,h})} \right) \\
    &\qquad \cdot \left( u_{h,s_{t,h},\pi_{t,h}(s_{t,h})}^{\hat q_t} - u_{h,s_{t,h},\pi_{t,h}(s_{t,h})}^{p^{*}} \right) \\
    +& 2 \sqrt{2} e^2 \beta_T(\delta) \sqrt{d} H^{3/2} \sqrt{\log(1+T)} \left[ \Reg_T(\MDP) + \sqrt{2} H^{3/2} \sqrt{T \log(2/\delta)} \right]^{1/2} \\
    +& H^{3/2} \sqrt{2T\log(2/\delta)} \,.
\end{align*}
\normalsize

\textbf{Step 6: Bounding the sum of $ p_h^{*}(s'|s_{t,h}, \pi_{t,h}(s_{t,h})) \llVert \phi_{s,a}^{s'} - \bar \phi_{t,h} \rrVert_{\cH_{t,h}^{-1}}^2 $.}

We first use the self-concordance property of the logistic loss. 
The logistic loss is a generalized self-concordant function with a self-concordance parameter of $3\sqrt{2}$ \citep[Proposition~1]{li2024provably}. Our initial exploration phase guarantees that the confidence set $\Theta_h$ has a diameter at most $\diam_{\cA,\cS,h}(\Theta_h) \le \frac{1}{3\sqrt{2}}$ (Lemma~\ref{lemma:constant diameter}), and given that both $\theta_h^*$ and $\htheta_{t,h}$ belong to $\Theta_h$, we can apply \citep[Proposition~8]{sun2019generalized} to bound the ratio of the Hessian matrices. For all $t,h \ge 1$, we obtain:
\begin{equation*}
    \cH_{t,h}(\theta_h^*) \preccurlyeq \exp\left( 3\sqrt{2} \diam_{\cA,\cS,h}(\Theta_h) \right) \cH_{t,h} 
    \preccurlyeq e \cH_{t,h} 
\end{equation*}
where the second inequality is due to Lemma~\ref{lemma:constant diameter}.
Thus for all $t,h\ge 1$, we have
\begin{equation*}
    \llVert \phi_{s,a}^{s'} - \bar \phi_{t,h} \rrVert_{\cH_{t,h}^{-1}}^2
    \le e \llVert \phi_{s,a}^{s'} - \bar \phi_{t,h} \rrVert_{\cH_{t,h}^{-1}(\theta_h^*)}^2 \,.
\end{equation*}
We sum over $T$ and get:
\begin{multline*}
    \sumT \sumH \sum_{s'\in\cS_{t,h}} p_h^{*}(s'|s_{t,h}, \pi_{t,h}(s_{t,h})) \llVert \phi_{s,a}^{s'} - \bar \phi_{t,h} \rrVert_{\cH_{t,h}^{-1}}^2 \\
    \le e \sumT \sumH \sum_{s'\in\cS_{t,h}} p_h^{*}(s'|s_{t,h}, \pi_{t,h}(s_{t,h})) \llVert \phi_{s,a}^{s'} - \bar \phi_{t,h} \rrVert_{\cH_{t,h}^{-1}(\theta_h^*)}^2 \,.
\end{multline*}
We apply \cite[Lemma~6~(I)]{li2024provably} and get 
\begin{equation*}
     \sumT \sumH \sum_{s'\in\cS_{t,h}} p_h^{*}(s'|s_{t,h}, \pi_{t,h}(s_{t,h})) \llVert \phi_{s,a}^{s'} - \bar \phi_{t,h} \rrVert_{\cH_{t,h}^{-1}}^2 \le 2e dH \log\left( 1 + T\right) \,. 
\end{equation*}

\textbf{Step 7: Bounding the sum of $ \nabla^2 \Gamma_{t,h,s_{t,h},\pi_{t,h}(s_{t,h})} \left( \bar u_{h,s_{t,h},\pi_{t,h}(s_{t,h})} \right) $.}

This term can be bounded as in \cite[Equation~(D.10)]{lee2025improved}. We obtain for all $h\ge 1$:
\begin{multline*}
    \sumT \dfrac{1}{2} \left( u_{h,s_{t,h},\pi_{t,h}(s_{t,h})}^{q_t^*} - u_{h,s_{t,h},\pi_{t,h}(s_{t,h})}^{p^{*}} \right)^\top \nabla^2 \Gamma_{t,h,s_{t,h},\pi_{t,h}(s_{t,h})} \left( \bar u_{h,s_{t,h},\pi_{t,h}(s_{t,h})} \right) \\
    \cdot \left( u_{h,s_{t,h},\pi_{t,h}(s_{t,h})}^{q_t^*} - u_{h,s_{t,h},\pi_{t,h}(s_{t,h})}^{p^{*}} \right) 
    \le 20\kappa \beta_T(\delta)^2 d H^2 \log \left( 1 + \dfrac{T}{d\lambda} \right) 
\end{multline*}
where the $H^2$ term results from the fact that the reward, given by the state value function, takes values in $[0,H]$.

\textbf{Step 8: Putting everything together.}

We use the fact that $ x^2 - bx -c \le 0 \implies x^2 \le b^2 + 2c $ with $ x^2 = \Reg_T(\MDP) $ and get with probability $1-5\delta$:
\begin{align*}
     \Reg_T(\MDP) \le 
     & 4 \sqrt{2e} B d H^2 \bar\sigma_T \sqrt{T} \log(1+T) \sqrt{\log(T/\delta)} \\
     +& H^{3/2} \sqrt{2T\log(2/\delta)} \\
     +& 8 \sqrt{e \nu_{\max}} B d H^{5/4} T^{1/4} \log(1+T) \log(T/\delta) \sqrt{\log(2/\delta)} \\
     +&8e^2 B d H^{9/4} T^{1/4} \sqrt{\log(1+T)\log(T/\delta)\log(2/\delta)} \\
     +&32e B^2 \nu_{\max} d^2 H \log(1+T) \log(T/\delta) \\
     +& 32e^2 B^2 d^2 H^3 \log^2(1+T) \log(T/\delta) \\
     +& 40\kappa B^2 d^2 H^2 \log( 1 + T ) \log(T/\delta)  \,.
\end{align*}
where applying the union bound gives the result with probability $1-5\delta$.

\end{proof}

\section{Proof of Theorem~\ref{theorem:lower bound}}\label{appendix:proof lower bound}

In this section we exhibit an MNL mixture MDP for which $\bar\sigma_{T}$ is small in $\smash{H^{-1/2}}$. We consider the instance introduced by \cite[Section~4.1]{park2024infinite} to prove their lower bound, see Figure~\ref{figure: MDP high kappa star}. There are $H+2$ states $s_1, s_2, \dots, s_{H+1}, s_{H+2}$ where $s_{H+2}$ is an absorbing state. The action space is defined by $\cA = \{-1, 1 \}^{d-1}$. For each step $h\ge 1$ and action $a\in \cA$, the reward is given by 
$$ r_h(s_i, a) = \mathbf{1}[i=H+2] \,. $$
For all $h\ge 1$, we simplify the setting so that all model parameters $ \theta^* $  are the same and in the space
$$ \theta^* \in \{-\bar \Delta, \bar \Delta\}^{d-1} \quad\text{ where }\quad \bar \Delta := \dfrac{B}{d-1} \log\left( \dfrac{ (1-\delta)(\delta + (d-1)\Delta) }{ \delta (1-\delta - (d-1)\Delta) } \right)  $$
with $\delta = 1/H$ and $ \Delta = 1/(4\sqrt{2HT}) $.

\begin{figure}[htbp]
    \centering
    \resizebox{\textwidth}{!}{%
    \begin{tikzpicture}[
        >={Stealth[scale=1.2]},
        node distance=3.5cm and 2.5cm, 
        state/.style={circle, draw, minimum size=1.2cm, thick, fill=white},
        dots/.style={draw=none, minimum size=1.2cm, font=\Large}
    ]
        \node[state] (x1) {$s_1$};
        \node[state, right=of x1] (x2) {$s_2$};
        \node[dots, right=of x2] (dots) {$\cdots$};
        \node[state, right=of dots] (xH) {$s_H$};
        \node[state, right=of xH] (xH1) {$s_{H+1}$};
        
        \node[state, below=2.5cm of dots] (xH2) {$s_{H+2}$};
        
        \draw[->, thick] (x1)   to[bend left=30] node[above, xshift=-0.5cm] {\tiny $ \dfrac{(H-1) \exp(-(\theta^*)^\top a)}{1+(H-1) \exp(-(\theta^*)^\top a)}$} (x2);
        \draw[->, thick] (x2)   to[bend left=30] node[above, xshift=-0.1cm] {\tiny  $ \dfrac{(H-1) \exp(-(\theta^*)^\top a)}{1+(H-1) \exp(-(\theta^*)^\top a)}$} (dots);
        \draw[->, thick] (dots) to[bend left=30] node[above, xshift=0.1cm] {\tiny  $\dfrac{(H-1) \exp(-(\theta^*)^\top a)}{1+(H-1) \exp(-(\theta^*)^\top a)}$} (xH);
        \draw[->, thick] (xH)   to[bend left=30] node[above, xshift=0.5cm] {\tiny  $\dfrac{(H-1) \exp(-(\theta^*)^\top a)}{1+(H-1) \exp(-(\theta^*)^\top a)}$} (xH1);
        
        \draw[->, thick] (x1) -- node[pos=0.25, below left] {\tiny $\dfrac{1}{1+(H-1) \exp(-(\theta^*)^\top a)}$} (xH2);
        \draw[->, thick] (x2) -- node[pos=0.3, right, xshift=2mm] {\tiny $\dfrac{1}{1+(H-1) \exp(-(\theta^*)^\top a)}$} (xH2);
        \draw[->, thick] (xH) -- node[pos=0.3, right] {\tiny $\dfrac{1}{1 +(H-1) \exp(-(\theta^*)^\top a)}$} (xH2);
        
        \draw[->, thick] (xH1) edge[loop below] node[below, align=center] {reward $= 0$} (xH1);
        \draw[->, thick] (xH2) edge[loop below] node[below, align=center] {reward $= 1$} (xH2);
        
    \end{tikzpicture}%
    }
    \caption{Illustration of the MDP instance we use in our regret lower bound.}
    \label{figure: MDP high kappa star}
\end{figure}
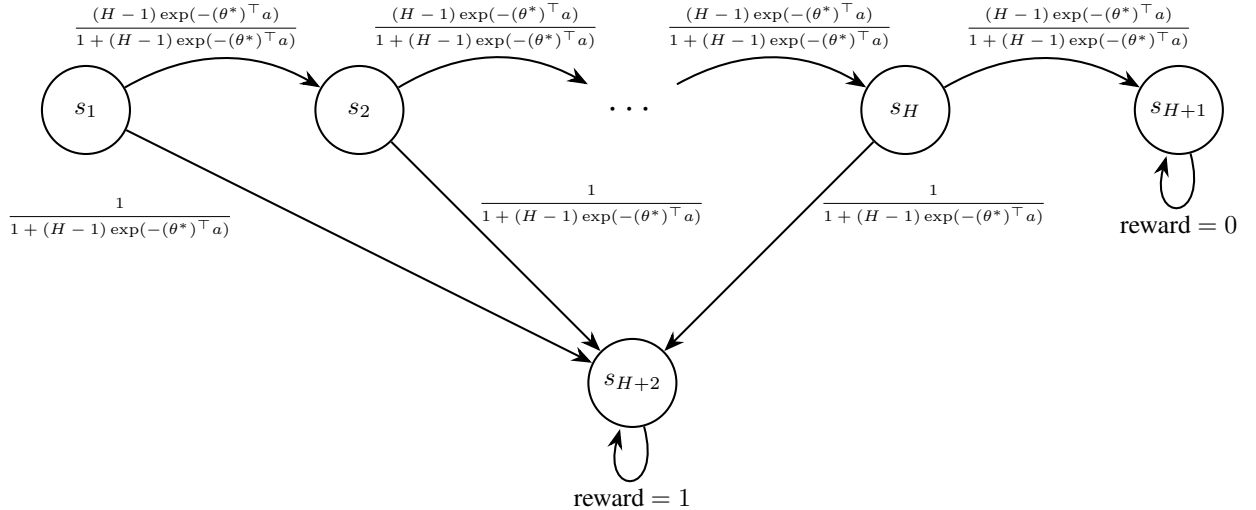

\TheoremLowerBound*

\begin{proof}
We show that the instance of MNL mixture MDP we introduced leads to $ \bar\sigma_{T}^2 $ upper bounded by $1/H$. Note that for any $h\ge 2$ the current state is either $ s_h $ or $s_{H+2}$. Because $s_{H+2}$ is absorbing, the state-value function has no variance in this state. Therefore if the current state is $s_{H+2}$, we have $\Var_{p^*} [V_{h+1}^*(s')|s_{H+2}] = 0$. We now consider the case when the current state is $s_h$ with $h\le H$.
Let us compute the state-value function in both cases:
\begin{align*}
    V_{h+1}^* (s_{H+2}) &= H-h \\
    V_{h+1}^* (s_{h+1}) &= \sum_{j=h+1}^H p^{*} (1-p^{*})^{j-h-1} (H-j)
\end{align*}
where $ p^{*} := p^{*}(s_{H+2}|s_j, \pi_j^*(s_j)) $ which is the same quantity for all $j\in \llbracket H \rrbracket$. Let us define $ \Delta V_{h+1}^* := V_{h+1}^* (s_{H+2}) - V_{h+1}^*(s_{h+1}) $. We have 
\begin{equation*}
    \Delta V_{h+1}^* = H - h - \sum_{j=h+1}^H p^{*} (1-p^{*})^{j-h-1} (H-j) = H-h - \sum_{j=0}^{H-h-1} p^{*} (1-p^{*})^j (H-h-1-j) \,.
\end{equation*}
We apply Lemma~\ref{lemma: modified geometric sum} with $ N = H-h $ and $ q = 1- p^{*} $ to get
\begin{equation*}
    \Delta V_{h+1}^* = \dfrac{1- (1-p^{*})^{H-h}}{1-(1-p^{*})} = \dfrac{1- (1-p^{*})^{H-h}}{p^{*}} \,.
\end{equation*}
Now note that for a binary outcome the variance can be written 
\begin{equation*}
    \Var_{p^*} [ V_{h+1}^*(s') ] = p^{*} (1-p^{*}) (\Delta V_{h+1}^*)^2 \,.
\end{equation*}
Consequently, the normalised variance becomes
\begin{equation*}\label{eq: example kappa star expression kappa}
    \frac{1}{H^2} \Var_{p^*} [ V_{h+1}^*(s') ] = \dfrac{1}{H^2} p^{*} (1-p^{*}) (\Delta V_{h+1}^*)^2 = \dfrac{1-p^{*}}{H^2 p^{*}} \left( 1 - (1-p^{*})^{H-h} \right)^2 \,.
\end{equation*}
Using the Bernoulli inequality we have that $ 1 - (1-p^{*})^{H-h} \le (H-h) p^{*} \le H p^{*} $. We substitute and get 
\begin{equation}\label{eq:lower bound 1}
    \frac{1}{H^2} \Var_{p^*} [ V_{h+1}^*(s') ] \le \dfrac{1-p^{*}}{H^2 p^{*}} H^2 (p^{*})^2 = (1-p^{*}) p^{*} \le p^{*} \,.
\end{equation}
Using the expression of $p^*$ we have
\begin{multline*}
    p^{*} = \frac{1}{1 + (H-1) \exp(-(\theta^*)^\top a^*)} 
    \le \frac{1}{1 + (H-1) \exp((d-1)\bar\Delta)} \\
    = \delta + (d-1) \Delta
    = \frac{1}{H} + \dfrac{(d-1)}{4\sqrt{2HT}} 
    \le \frac{2}{H} 
\end{multline*}
where the first inequality is due to \cite[Equation~(10)]{park2024infinite}, the second equality to \cite[Lemma~G.1]{park2024infinite} and the last inequality is due to $ T \ge (d-1)^2 H^3 /32 $.
Consequently, by substituting into Equation~\eqref{eq:lower bound 1}, we can upper bound the problem-dependent constant $\bar \sigma_T$ by:
\begin{equation}\label{eq:lower bound 2}
    \bar\sigma_T \le \sqrt{ (TH)^{-1} \sumT \sumH \frac{2}{H}} = \sqrt{2} H^{-1/2} \,.
\end{equation}
Moreover, for all $h\ge1$ we have $\lVert\theta_h^*\rVert_2 \le B \le \tfrac{3}{2} (1 + \log(H-1)) $ \citep[Theorem~3]{park2024infinite}.
This MNL mixture MDP instance therefore yields a regret lower bound of
\begin{equation*}
    \Reg_T \ge \Omega\left( d H^{3/2} \sqrt{T} \right) 
    \ge \Omega\left( B d H^{3/2} \sqrt{T} \right)
    \overset{\eqref{eq:lower bound 2}}{\ge} \Omega\left( d H^{2} \bar \sigma_T \sqrt{T} \right) 
\end{equation*}
where the first inequality is by \citep[Theorem~3]{park2024infinite}.
\end{proof}

\section{KL-Constrained Robust MDPs}\label{appendix:KLCR MDP}
In this section we recall and prove Proposition~\ref{prop:regret bound KLCR MDP}.

\PropositionRegretKLCRMDP*

\begin{proof}
Let $s,a\in \cS\times \cA$ and $h\ge 1$.
As the variance and the softmax are translation invariant, we can assume without loss of generality that $\min_{s'\in\cS_{h,s,a}} V_{h+1}^*(s') = 0$. Thus $Z:= \sum_{s'\in\cS_{h,s,a}} \exp(-\eta V_{h+1}^*(s')) \ge 1 $.
For any threshold $v>0$, let $\cS_v := \{ s' \in \cS_{h,s,a} | V_{h+1}^*(s') \ge v \}$. Then
\begin{equation*}
    \Pr[ V_{h+1}^*(s') \ge v | s,a ] = \frac{1}{Z} \sum_{s''\in\cS_v} \exp(-\eta V_{h+1}^*(s''))
    \le  \sum_{s''\in\cS_v} \exp(-\eta V_{h+1}^*(s''))
    \le U \exp(-\eta v) 
\end{equation*}
using $Z\ge 1$ and $|\cS_v|\le U$.
In particular $\Pr[ V_{h+1}^*(s') \ge v | s,a ] \le 1$ as soon as $v> \log U / \eta$.
Using the layer-cake representation
\begin{equation*}
    \E_{s'\sim p^*} [ (V_{h+1}^*(s') )^2 | s,a ] = \int_0^\infty 2 v  \Pr\left[ V_{h+1}^*(s') \ge v | s,a \right] dv \,.
\end{equation*}
Set $v_0 = \log U / \eta$ and split the integral:
\begin{equation*}
    \E_{s'\sim p^*} [ (V_{h+1}^*(s') )^2 | s,a ] = \int_0^{v_0} 2 v  \Pr\left[ V_{h+1}^*(s') \ge v | s,a \right] dv 
    + \int_{v_0}^\infty 2 v  \Pr\left[ V_{h+1}^*(s') \ge v | s,a \right] dv \,.
\end{equation*}
For the first piece, use $\Pr[ V_{h+1}^*(s') \ge v | s,a ] \le 1$:
\begin{equation*}
    \int_0^{v_0} 2 v dv = v_0^2 = \frac{\log^2 U}{\eta^2} \,.
\end{equation*}
For the second piece, use $\Pr[ V_{h+1}^*(s') \ge v | s,a ] \le U \exp(-\eta v)$:
\begin{equation*}
    2U \int_{v_0}^\infty v \exp(-\eta v) dv = 2U \frac{\exp(-\eta v_0)}{\eta} \left( v_0 + \frac{1}{\eta} \right)
    = \frac{2}{\eta} \left( \frac{\log U}{\eta} + \frac{1}{\eta} \right)
\end{equation*}
since $ \exp(-\eta v_0) = 1/U $.
By combining both pieces of the integral and using $\Var_{s'\sim p^*} [ V_{h+1}^*(s')  | s,a ] \le \E_{s'\sim p^*} [ (V_{h+1}^*(s') )^2 | s,a ]$, we have:
\begin{equation*}
    \Var_{s'\sim p^*} [ (V_{h+1}^*(s') ) | s,a ] \le \frac{\log^2 U + 2\log U + 2 }{\eta^2} \,.
\end{equation*}
Thus Theorem~\ref{theorem:regret bound} yields with probability $1-5\delta$ a regret upper bounded by 
\begin{equation*}
    \Reg_T \le \tilde O \left( \frac{\log U d H}{\eta} \sqrt{T } \log \delta^{-1} \right) \,.
\end{equation*}

\end{proof}

\section{Bound on \texorpdfstring{$ \protect \bar\sigma_T $}{bar sigma} Under an Optimal Policy}\label{appendix:bound bar sigma}
In this section we bound the problem-dependent constant $\bar \sigma_T$ under an optimal policy $\pi^*$.

\begin{restatable}{prop}{}\label{prop:bound bar sigma optimal trajectory}
    For any MNL mixture MDP, the problem-dependent $\bar\sigma_T$ under an optimal policy $\pi^*$ is bounded by
    \begin{equation*}
        \E_{\pi^*}[\bar \sigma_T^2 ] \le \frac{1}{4H} \,.
    \end{equation*}
\end{restatable}

\begin{proof}
Let $G := \sumH r_h(s_h, a_h)$ be the total return of a trajectory $(s_h, a_h)_h$. 
By Popoviciu's inequality we have
\begin{equation}\label{eq:proof low sigma 1}
    \Var_{\pi^*, p^*} [G] \le \frac{H^2}{4} \,.
\end{equation}
Under $\pi^*$ and $p^*$, the process $M_h := V_h^*(s_h) + \sum_{j<h} r_j(s_j, a_j) $ is a martingale associated with the filtration $\cF_h := \{ s_1, a_1, \dots, s_h \}$.
The martingale increment is $\Delta_h = V_{h+1}^*(s_{h+1}) - \E_{s'\sim p^*} [V_{h+1}^*(s') | s_h, a_h] $.
We have that
\begin{align}
    \Var_{\pi^*, p^*} [G] &= \Var_{\pi^*, p^*}[ M_{H+1} - M_1 ] &&(M_1 \text{ is deterministic}) \nonumber \\
    &= \Var_{\pi^*, p^*} \left[ \sumH \Delta_h \right] \nonumber \\
    &= \sumH \Var_{\pi^*, p^*}[\Delta_h] +2 \sum_{h=1}^{H-1} \sum_{j=h+1}^H \mathrm{Cov}[\Delta_h, \Delta_{j}] \,. \label{eq:proof low sigma 2}
\end{align}
We study the covariance term, for martingale increments we have $\E[\Delta_h|\cF_h] = 0$, thus
\begin{equation*}
    \mathrm{Cov}[\Delta_h, \Delta_{j}] = \E[ \Delta_h \Delta_{j} ] = \E[\Delta_h \E[\Delta_{j}|\cF_{j}]] = 0 \,.
\end{equation*}
We substitute into Equation~\eqref{eq:proof low sigma 2} and get
\begin{equation*}
    \Var_{\pi^*, p^*} [G] = \sumH \Var_{\pi^*, p^*}[\Delta_h] \le \frac{H^2}{4} \,.
\end{equation*}
Now note that $\bar\sigma_T$ is defined as $\E_{\pi^*}[\bar\sigma_T^2] = (TH)^{-1} \sum_t H^{-2} \Var_{\pi^*, p^*} [G]$. Therefore
\begin{equation*}
    \E_{\pi^*}[\bar \sigma_T^2 ] \le \frac{1}{4H} \,.
\end{equation*}
\end{proof}

\section{Experiments}\label{appendix:experiments}
In this section, we evaluate Algorithm~\ref{algo:learning routine 2} numerically. We empirically validate the $\smash{H^{3/2}}$ rate in the regret bound predicted by Theorem~\ref{theorem:regret bound} and benchmark against the prior state-of-the-art.
We consider a challenging MNL mixture MDP family used by \cite{park2024infinite} to derive the $\smash{\Omega(dH^{3/2}\sqrt{T})}$ lower bound.
The structure of this MDP family is described in Appendix~\ref{appendix:proof lower bound}, see Figure~\ref{figure: MDP high kappa star}.
The instance has feature dimension $d=2$ and inhomogeneous transition parameters $\smash{ \theta_h^* \in \{ -\bar \Delta,  \bar \Delta \}^d }$ sampled independently per stage. 

We run our algorithm over 40 independent seeds with the parameters $\tau=80, \lambda_0=1, \lambda=10$ and a learning rate in the OMD procedure of $\eta=20$. We also shrink the size of the confidence ellipsoid $\cC_{t,h}(\delta)$ by introducing a scale parameter $\mathrm{beta\_scale}=0.01$ and set $\smash{\tilde\beta_t = \mathrm{beta\_scale} ( (\tfrac{2}{3} d \log(t/\delta))^{1/2} + 24 B \sqrt{d} ) }$. To tackle the non-concave maximisation problem in the optimistic reward definition (see Equation~\eqref{eq:definition optimistic reward}), we employ the Frank-Wolfe algorithm to find a local optimum within the confidence set.
We compare our algorithm to the \textsc{UCRL-MNL-OL} algorithm proposed by \cite{li2024provably}. We run it with the same seeds and parameters $ \lambda=10, \eta=10 $. We also use the scaling parameter $\mathrm{beta\_scale}=0.02$ and set $\smash{\tilde\beta_t = \mathrm{beta\_scale} \sqrt{d} \log U \log(tH/\delta) }$.

\begin{figure}[!ht]
    \centering
    \includegraphics[width=0.9\textwidth]{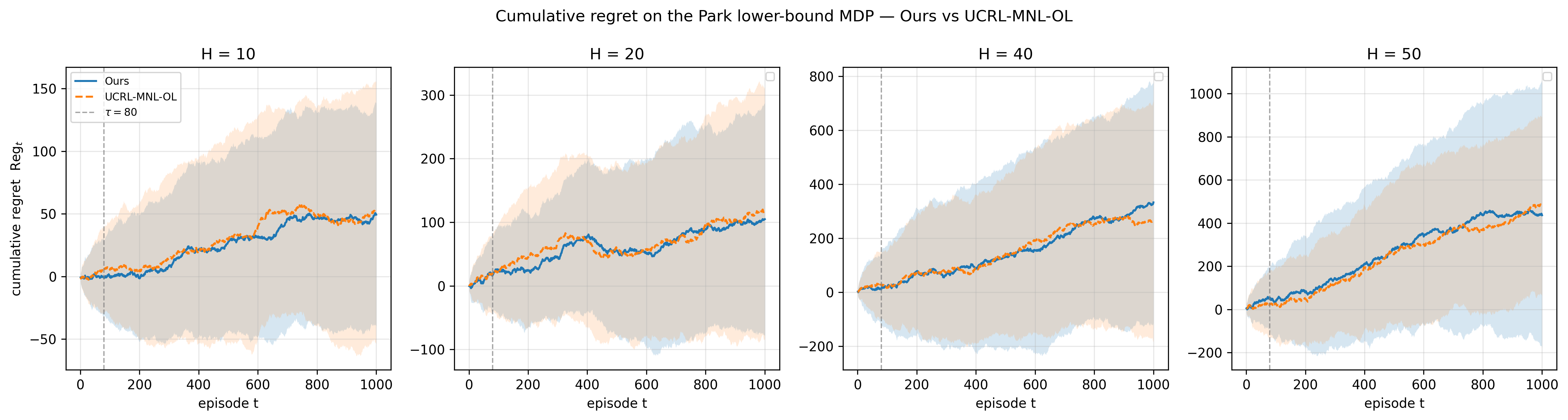}
    \caption[]{Cumulative regret $\Reg_T$ for different episode lengths $H$ as a function of the episode $t$. We plot the means $\pm 1$ standard deviation across 40 seeds.}
    \label{figure:experiment cumulative regret}
\end{figure}

We compute the per-episode regret as $\smash{ V_1^*(s_{t,1}) - \sum_h r_h(s_{t,h}, \pi_{t,h}(s_{t,h})) }$, i.e. the difference between the true expected return under the optimal policy and the sum of rewards along the realised trajectory. The per-episode regret can be negative due to trajectory randomness, but the cumulative regret is positive in expectation.
In Figure~\ref{figure:experiment cumulative regret}, we plot the cumulative regret $\Reg_t$ for different episode lengths $H$ as a function of the episode $t$. As both algorithms use the same OMD procedure, the regret curves have similar shapes. 
Interestingly, while the initial exploration phase was anticipated to be costly, it proved to be a highly effective strategy on this specific MDP instance.
In Figure~\ref{figure:experiment scaled regret}, we plot the cumulative regret $\Reg_T$ at $T=1000$ episodes normalised by $\smash{H^{3/2}}$ as a function of the horizon $H$. For both algorithms, the ratio $\smash{\Reg_T/H^{3/2}}$ is approximately constant, consistent with the $\smash{\tilde O(dH^{3/2}\sqrt{T})}$ regret scaling. 
While this was theoretically expected for our algorithm, it is less expected from the theoretical guarantees of the \textsc{UCRL-MNL-OL} algorithm \citep[Theorem~2]{li2024provably}.

These experiments suggest that both algorithms behave similarly, and that while the exploration phase may not be necessary in practicen it allows us to derive a sharper analysis.

\begin{figure}[!ht]
    \centering
    \includegraphics[width=0.8\textwidth]{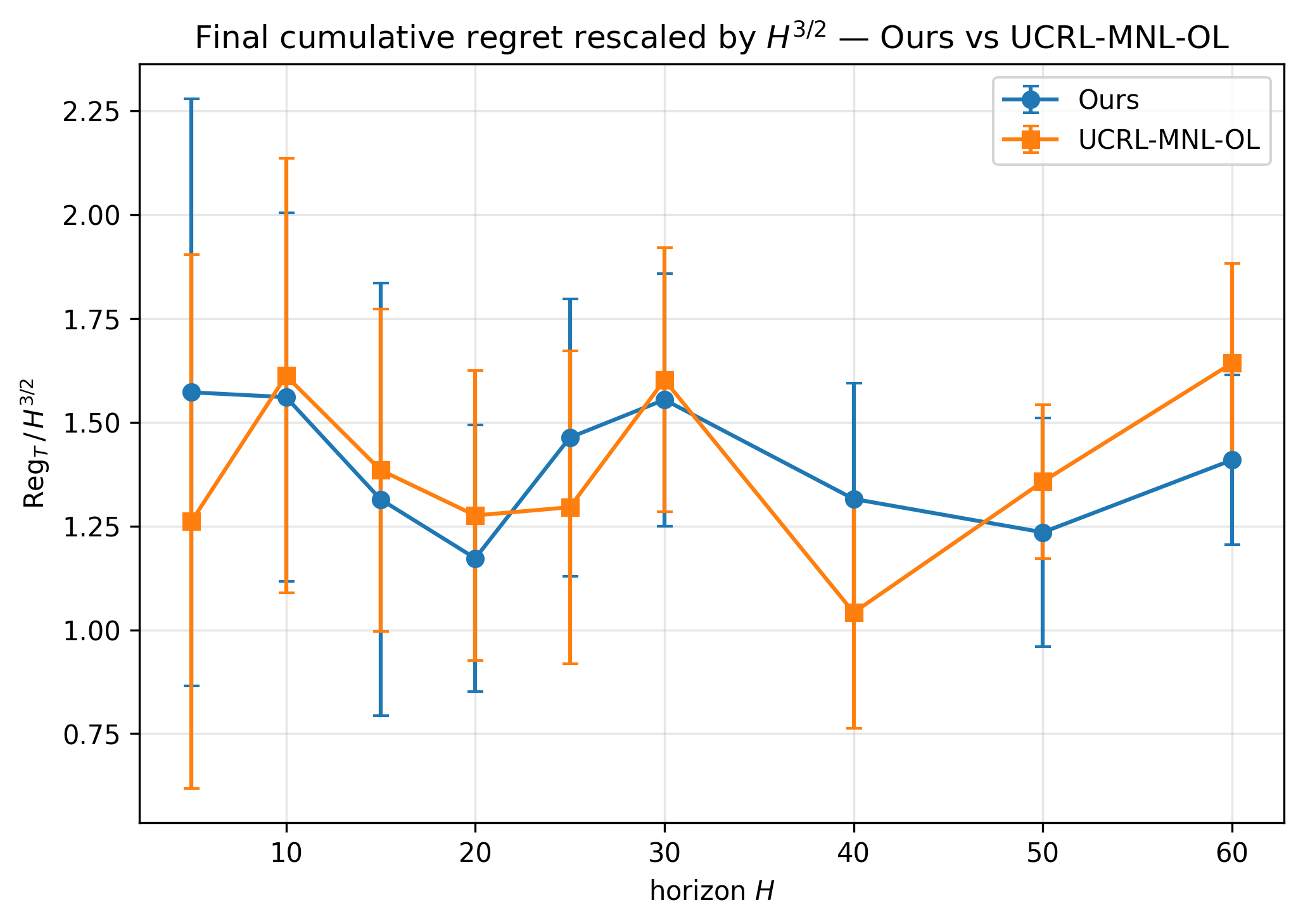}
    \caption[]{Cumulative regret $\Reg_T$ at $T=1000$ episodes normalised by $\smash{H^{3/2}}$ as a function of the horizon $H$. We plot the means $\pm 1$ standard deviation across 40 seeds.}
    \label{figure:experiment scaled regret}
\end{figure}

\section{Auxiliary Result}\label{appendix:auxiliary results}

\begin{restatable}{lmm}{}\label{lemma: modified geometric sum}
For any $N\ge 1$ and $q\in ]0, 1[$, we have
\begin{equation*}
    N-1 - (1-q) \sum_{k=0}^{N-1} q^k (N-1-k) = \dfrac{1 - q^N}{1-q} - 1 \,.
\end{equation*}
\end{restatable}

\begin{proof}
Let $N\ge 1$ and $ q \in ]0,1[ $. We define $ S := \sum_{k=0}^{N-1} q^k (N-1-k) $. We have
\begin{align*}
    S- qS &= \sum_{k=1}^{N-1} \left[ (N-1-k) - (N-1-(k-1)) \right] + N-1 \\
    &= - \sum_{k=1}^{N-1} q^k + N-1 \\
    (1-q) S &= -q \dfrac{1-q^{N-1}}{1-q} + N-1 \,.
\end{align*}
Thus we get
\begin{equation*}
    N-1 - (1-q) S = \dfrac{q-q^N}{1-q} = \dfrac{1-q^N}{1-q} - 1 
\end{equation*}
which concludes the proof.
\end{proof}

\end{document}